%% file: main.tex
\documentclass[preprint, 10pt]{article}

\input{packages}
\input{math_commands.tex}

\input{definitions}

\def\month{08}  
\def\year{2024}
\def\openreview{\url{https://openreview.net/forum?id=gfANevPraH}} 

\title{Demonstrating and Reducing Shortcuts in Vision-Language Representation Learning}

\input{authors}

\begin{document}

\maketitle

\begin{abstract}
\input{sections/00_abstract}
\end{abstract}


\acresetall

\input{sections/01_introduction}
\input{sections/02_background}

\input{sections/03_shortcuts_setup}
\input{sections/04_shortcuts_results}

\input{sections/05_evaluation_setup}
\input{sections/06_evaluation_results}

\input{sections/07_related_work}
\input{sections/08_conclusion}
\input{sections/09_impact}

\section*{Acknowledgements}
We thank Marco Federici and Mathijs Henquet for the discussions on mutual information and feedback on the draft. Additionally, we thank Shashank Gupta and Panagiotis Efstratiadis for helpful feedback. 

This research was supported by the Nationale Politie, Ahold Delhaize, project IDEAS with project number VI.Vidi.223.166 of the NWO Talent Programme, which is (partly) financed by the Dutch Research Council (NWO), the Hybrid Intelligence Center, a 10-year program funded by the Dutch Ministry of Education, Culture and Science through the Netherlands Organisation for Scientific Research, \url{https://hybrid-intelligence-centre.nl}, project LESSEN with project number NWA.1389.20.183 of the research program NWA ORC 2020/21, which is (partly) financed by the Dutch Research Council (NWO), project ROBUST with project number KICH3.LTP.20.006, which is (partly) financed by the Dutch Research Council (NWO), DPG Media, RTL, and the Dutch Ministry of Economic Affairs and Climate Policy (EZK) under the program LTP KIC 2020-2023, and the FINDHR (Fairness and Intersectional Non-Discrimination in Human Recommendation) project that received funding from the European Union’s Horizon Europe research and innovation program under grant agreement No 101070212.
All content represents the opinion of the authors, which is not necessarily shared or endorsed by their respective employers and/or sponsors.

\bibliography{references}  
\bibliographystyle{tmlr}

\appendix
\input{sections/appendix/appendix}

\end{document}

%% file: packages.tex
\usepackage[dvipsnames]{xcolor}
\usepackage[accepted]{tmlr}
\usepackage{amsmath}
\usepackage[utf8]{inputenc} 
\usepackage[T1]{fontenc}    
\usepackage[colorlinks=true,urlcolor=blue,citecolor=blue]{hyperref}       
\usepackage{url}            
\usepackage{booktabs}       
\usepackage{amsfonts}       
\usepackage{nicefrac}       
\usepackage{microtype}      
\usepackage{cleveref}       
\usepackage{lipsum}         
\usepackage{graphicx}
\usepackage{acronym}
\usepackage{tikz}
\usepackage{bbm}
\usepackage[skip=2pt]{caption}
\usepackage{subcaption}
\usepackage{algorithm} 
\usepackage{algpseudocode}
\usepackage[toc,page]{appendix}
\usepackage{wrapfig}
\setcitestyle{sort, yearnat}
\usepackage[inline]{enumitem}
\usepackage[normalem]{ulem}
\usepackage{titlesec}
\usepackage{titlecaps}
\usepackage{fancyhdr}
\usepackage{pifont}
\usepackage{numprint}
\npthousandsep{,}
\npdecimalsign{.}
\usepackage{comment} 
\usepackage{longtable}

\usepackage{amsthm}
\usepackage{mathtools}
\usepackage{amssymb}

\newtheorem{assumption}{Assumption}

\newtheorem{innercustomthm}{Theorem}
\newenvironment{thm}[1]
{\renewcommand\theinnercustomthm{#1}\innercustomthm}
{\endinnercustomthm}

\usepackage{natbib}
\PassOptionsToPackage{sort&compress}{natbib} 

%% file: math_commands.tex
\usepackage{amsmath,amsfonts,bm}

\def\eqref#1{equation~\ref{#1}}

\def\1{\bm{1}}

\newcommand\infonce[1]{\ensuremath{\mathcal{L}_{\text{InfoNCE}}^{#1}}}
\newcommand\ltd{\ensuremath{\mathcal{L}_{\text{InfoNCE+LTD}}}}
\newcommand\ifm{\ensuremath{\mathcal{L}_{\text{InfoNCE+IFM}}}}

\DeclareMathAlphabet{\mathsfit}{\encodingdefault}{\sfdefault}{m}{sl}
\SetMathAlphabet{\mathsfit}{bold}{\encodingdefault}{\sfdefault}{bx}{n}

\def\sR{{\mathbb{R}}}



%% file: definitions.tex
\AtBeginDocument{%
  \providecommand\BibTeX{{%
    \normalfont B\kern-0.5em{\scshape i\kern-0.25em b}\kern-0.8em\TeX}}}

\acrodef{ICR}{image-caption retrieval}
\acrodef{SSL}{self-supervised learning}
\acrodef{CL}{contrastive learning}
\acrodef{MS-COCO}{MS-COCO Captions}
\acrodef{Flickr30k}{Flickr30k}
\acrodef{i2t}{image-to-text}
\acrodef{t2i}{text-to-image}
\acrodef{VL}{vision-language}
\acrodef{VLM}{vision-language model}
\acrodef{SOTA}{state-of-the-art}
\acrodef{MI}{mutual information}
\acrodef{LTD}{latent target decoding}
\acrodef{IFM}{implicit feature modification}
\acrodef{SBERT}{Sentence-BERT}
\acrodef{SVL}{synthetic shortcuts for vision-language}
\acrodef{i2t}{image-to-text}
\acrodef{t2i}{text-to-image}

\definecolor{cblue}{RGB}{0, 83, 155}
\definecolor{cred}{RGB}{238, 49, 42}
\definecolor{cyellow}{RGB} {255, 223, 0}
\definecolor{corange}{RGB} {245, 132, 43}
\definecolor{cpurple}{RGB}{123, 67, 154}
\definecolor{clightblue}{RGB}{0, 174, 239}
\definecolor{cburgundy}{RGB}{198, 12, 70}
\definecolor{cgreen}{RGB}{0, 158, 71}
\definecolor{cpurpler}{RGB}{158, 32, 110}
\definecolor{update_green}{RGB}{0, 153, 0}

\definecolor{baseline}{rgb}{0.18, 0.55, 0.34}
\definecolor{cap_only}{rgb}{0.8, 0.52, 0.25}
\definecolor{img_only}{rgb}{0.5, 0.0, 0.0}
\definecolor{us_with}{rgb}{0.539361783929258, 0.31820069204152246, 0.6458285274894272}
\definecolor{us_without}{rgb}{0.6054901960784314, 0.7074509803921569, 0.8392156862745098}
\definecolor{bs_with}{rgb}{0.0196078431372549, 0.18823529411764706, 0.3803921568627451}
\definecolor{bs_without}{rgb}{0.974163783160323, 0.7508650519031141, 0.6424452133794694}
\definecolor{improve}{rgb}{152.9, 255.0, 77.5}

\newcommand{\xmark}{\ding{55}}%

\newcommand{\headernodot}[1]{\vspace{1mm}\noindent\textbf{#1}}
\newcommand{\header}[1]{\headernodot{#1.}}

\newcommand{\img}[1]{\mathbf{x}_{\mathcal{I}}^{#1}}
\newcommand{\capt}[2]{\mathbf{x}_{\mathcal{C}_{#2}}^{#1}}
\newcommand{\zimg}[2]{\mathbf{z}_{\mathcal{I}_{#2}}^{#1}}
\newcommand{\zcapt}[2]{\mathbf{z}_{\mathcal{C}_{#2}}^{#1}}

\newcommand{\zimgsuff}[1]{\mathbf{z}_{ \mathcal{I} \rightarrow \mathcal{C}_{#1}}^{\textit{SUF}}}
\newcommand{\zcaptsuff}[2]{\mathbf{z}_{\mathcal{C}_{#2} \rightarrow \mathcal{I}}^{#1{\textit{SUF}}}}
\newcommand{\zminsuff}[1]{\mathbf{z}_{ \mathcal{I} \rightarrow \mathcal{C}_{#1}}^{\textit{MIN}}}
\newcommand{\zoptimal}[2]{\mathbf{z}_{\mathcal{I} \rightarrow #2}^{\textit{OPT}}{#1}}

\newcommand{\sct}{S_{SynSC}}

\newtheorem{definition}{Definition}[section]
\newcommand{\positive}{\textcolor{ForestGreen}{\boldsymbol{\uparrow}}}
\newcommand{\negative}{\textcolor{Red}{\boldsymbol{\downarrow}}}

\makeatletter
\newcommand*{\rom}[1]{\expandafter\@slowromancap\romannumeral #1@}
\makeatother

 \newcommand{\baseline}{
	\tikz[line width = 0pt, line join = round,]{
		\draw[fill=baseline,baseline] (0,0) rectangle (5.0mm,1.7mm);
	}
}

 \newcommand{\caponly}{
	\tikz[line width = 0pt, line join = round,]{
		\draw[fill=cap_only,cap_only] (0,0) rectangle (5.0mm,1.7mm);
	}
}

 \newcommand{\imgonly}{
	\tikz[line width = 0pt, line join = round,]{
		\draw[fill=img_only,img_only] (0,0) rectangle (5.0mm,1.7mm);
	}
}

 \newcommand{\uswith}{
	\tikz[line width = 0pt, line join = round,]{
		\draw[fill=us_with,us_with] (0,0) rectangle (5.0mm,1.7mm);
	}
}

 \newcommand{\uswithout}{
	\tikz[line width = 0pt, line join = round,]{
		\draw[fill=us_without,us_without] (0,0) rectangle (5.0mm,1.7mm);
	}
}

 \newcommand{\bswith}{
	\tikz[line width = 0pt, line join = round,]{
		\draw[fill=bs_with,bs_with] (0,0) rectangle (5.0mm,1.7mm);
	}
}

 \newcommand{\bswithout}{
	\tikz[line width = 0pt, line join = round,]{
		\draw[fill=bs_without,bs_without] (0,0) rectangle (5.0mm,1.7mm);
	}
}

%% file: authors.tex
\author{%
\name Maurits Bleeker\thanks{Co-first author.} \email m.j.r.bleeker@uva.nl \\
\addr University of Amsterdam, Amsterdam, The Netherlands
\AND
\name Mariya Hendriksen\footnotemark[1] \email m.hendriksen@uva.nl \\ 
\addr AIRLab, University of Amsterdam, Amsterdam, The Netherlands
\AND
\name Andrew Yates \email a.c.yates@uva.nl \\
\addr University of Amsterdam, Amsterdam, The Netherlands
\AND
\name Maarten de Rijke \email m.derijke@uva.nl \\
\addr University of Amsterdam, Amsterdam, The Netherlands
}

%% file: sections/00_abstract.tex

\Acfp{VLM} mainly rely on contrastive training to learn general-purpose representations of images and captions.
We focus on the situation when one image is associated with several captions, each caption containing both information shared among all captions and unique information per caption about the scene depicted in the image.
In such cases, it is unclear whether contrastive losses are sufficient for learning task-optimal representations that contain all the information provided by the captions or whether the contrastive learning setup encourages the learning of a simple shortcut that minimizes contrastive loss.
We introduce \textit{\acl{SVL}}: a training and evaluation framework where we inject synthetic shortcuts into image-text data.
We show that contrastive \ac{VLM}s trained from scratch or fine-tuned with data containing these synthetic shortcuts mainly learn features that represent the shortcut. 
Hence, contrastive losses are not sufficient to learn task-optimal representations, i.e., representations that contain all task-relevant information shared between the image and associated captions.
We examine two methods to reduce shortcut learning in our training and evaluation framework:
\begin{enumerate*}[label=(\roman*)]
\item \acl{LTD} and 
\item \acl{IFM}.
\end{enumerate*}
We show empirically that both methods improve performance on the evaluation task, but only partially reduce shortcut learning when training and evaluating with our shortcut learning framework.
Hence, we show the difficulty and challenge of our shortcut learning framework for contrastive \acl{VL} representation learning.

%% file: sections/01_introduction.tex

\section{Introduction}
\label{sec:introduction}

Recent work on understanding the internal mechanisms of representation learning has brought to attention the problem of shortcut learning~\citep{robinson2021can, chen2021intriguing, scimeca2022which}.
While there are multiple definitions of shortcut learning \citep[e.g., ][]{geirhos2020shortcut, wiles_2022_afinegrained}, in this work we define \emph{shortcuts} as \emph{easy-to-learn discriminatory features that minimize the (contrastive) optimization objective but are not necessarily sufficient for solving the evaluation task}. 
More specifically, we focus on the problem of shortcut learning in the relatively unexplored context of \ac{VL} representation learning with multiple matching captions per image.

\Ac{CL} plays a crucial role in \ac{VL} representation learning. 
Despite the success of non-contrastive approaches, e.g., \citep{bardes2022vicreg}, the dominant paradigm in \ac{VL} representation learning revolves around either fully contrastive strategies~\citep{faghri2018improving, li2019visual, jia2021scaling, radford2021learning} or a combination of contrastive methods with additional objectives~\citep{li2021align, zeng2022multi, li2022blip, zeng2022multi, li2023blip}.
It is standard practice in contrastive \ac{VL} representation learning to sample batches of image-caption pairs and maximize the alignment between the representations of the matching images and captions \citep{radford2019language, jia2021scaling}. 
Given that the typical \ac{VL} benchmarks, e.g., \acl{Flickr30k}\acused{Flickr30k}~\citep{young2014image} and \acl{MS-COCO}\acused{MS-COCO}~\citep{lin2014microsoft, chen2015microsoft}, are constructed in such a way that each image is associated with multiple captions, each caption can be seen as a different \textit{view} of the image it describes. 
Therefore, \ac{CL} with multiple captions per image can be seen as \ac{CL} with multiple views, where each caption provides a different view of the scene depicted in the image.

\Ac{CL} with multiple views, where each view represents a different observation of the same datapoint, has proven to be effective for general-purpose representation learning~\citep{hjelm2019learning, chen2020simple, tian2020contrastive}.
The goal of multi-view (contrastive) representation learning methods is to learn representations that remain invariant to a shift of view, which is achieved by maximizing alignment between embeddings of similar views. 
A core assumption within the multi-view representation learning literature is that task-relevant information is shared across views whereas task-irrelevant information is not shared, given a downstream evaluation task~\citep{zhao2017multi, federici2020learning, tian2020contrastive, shwartz2023compress}.

\begin{wrapfigure}[27]{r}{0.4\textwidth}
	\centering
	\includegraphics[width=0.4\textwidth]{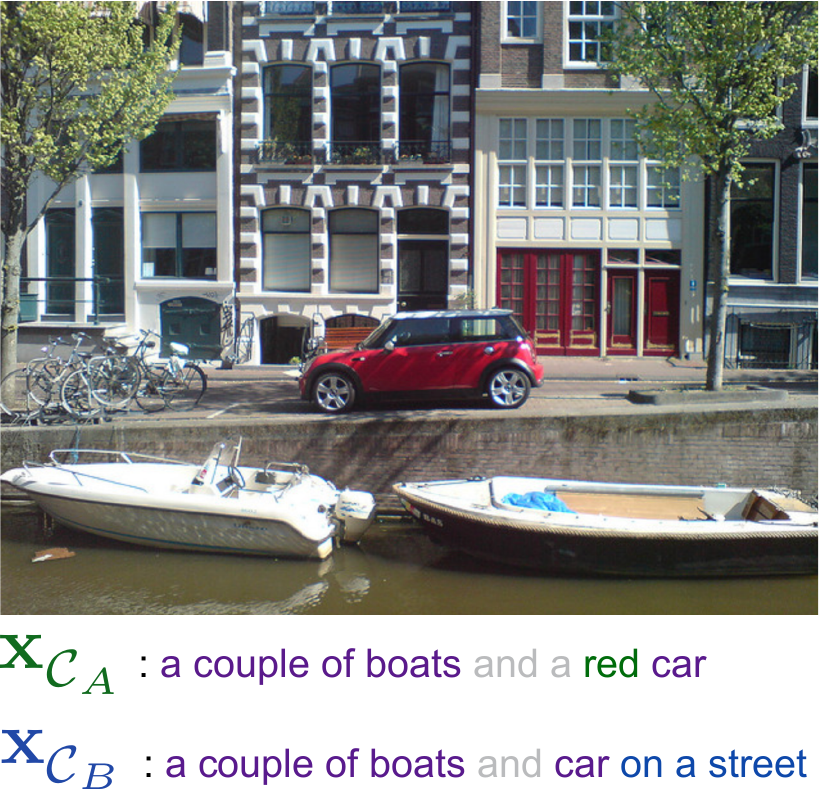}
	\caption{
	Shared vs. caption-specific information given an example of one image and two associated captions $\capt{}{A}$ and $\capt{}{B}$.
	The purple color indicates information shared between the image and both captions.
	The green color indicates task-relevant information specific for $\capt{}{A}$.
	The blue color indicates task-relevant information specific for $\capt{}{B}$.
	}
	\label{fig:shared-vs-specific}
\end{wrapfigure}

An open challenge in the multi-view representation learning domain concerns \emph{learning representations that contain task-relevant information that is not shared among different views, i.e., that may be unique for some views}~\citep{shwartz2023compress, zong2023self}.
In the case of image-caption datasets where each image is paired with at least one corresponding caption, the captions matching the same image do not necessarily share the same information as each caption is distinct and may describe different aspects of the image~\citep{biten2022_is}.
Figure~\ref{fig:shared-vs-specific} illustrates the concept of shared vs.\ caption-specific task-relevant information.
The image is accompanied by two captions: `a couple of boats and a red car' ($\capt{}{A}$) and `a couple of boats and a car on a street' ($\capt{}{B}$).
The shared information between the captions includes `couple of boats' and `car'.
Caption $\capt{}{A}$ provides unique information by describing the car as `red'. Caption $\capt{}{B}$ adds unique contextual details about the location with the phrase `on a street'.
To learn task-optimal representations, it is essential to integrate both the shared and unique information from these captions.
Furthermore, given the typical quality of captions of image-caption datasets~\citep{chen2015microsoft}, we assume that all information present in the captions is relevant.
Hence, each image-caption pair may contain both \emph{shared} task-relevant information, i.e., information shared across all the captions in the tuple, and \emph{unique} task-relevant information, i.e., information not shared with other captions.
Therefore, learning task-optimal representations for the image implies learning all task-relevant information that comprises both shared and caption-specific information.

Another problem of \ac{CL} approaches is related to \emph{feature suppression}. 
\cite{shwartz2023compress} argue that although contrastive loss functions lack explicit information-theoretical constraints aimed at suppressing non-shared information among views, the learning algorithm benefits from simplifying representations by suppressing features from the input data that are not relevant for minimizing the contrastive loss.
Furthermore, \citet{robinson2021can} demonstrate that contrastive loss functions are susceptible to solutions that suppress features from the input data.
In the case of \ac{VL}, \ac{CL} with multiple captions per image where at least one caption contains caption-specific information, the image representation can never have a perfect alignment with all matching captions.
This is due to the misalignment that happens when encoding unique information for the other captions.
Therefore, it is unclear whether contrastive methods can learn task-optimal representations, i.e., representations that contain all information present in the captions associated with the image, or if they learn only the minimal shared information, i.e., information shared between the image and all captions that are sufficient to minimize the contrastive discrimination objective. 
An illustration of minimal shared information and a task-optimal representation is given in Figure~\ref{fig:latent_viz}. 

\input{figures/latent_viz/latent_viz} 

Motivated by the abovementioned problems, we address the following question:
\begin{quote}
\emph{In the context of \ac{VL} representation learning with multiple captions per image, to what extent does the presence of a shortcut hinder learning task-optimal representations?}
\end{quote}

To answer this question, we investigate the problem of shortcut learning for \ac{VL} representation learning with multiple captions per image.
We do this by introducing the \acfi{SVL} framework for adding additional, easily identifiable information to image-caption tuples. 
The information that we add is represented as identifiers that are applied to both image and caption; these identifiers do not bear any semantic meaning. 
The identifiers provide additional shared information between the image and captions, which is a subset of the total shared information between the image and the caption.
For details and examples of shortcuts, refer to Section~\ref{eval:method}, where Figure~\ref{fig:shortcut-example} illustrates an example of an image-caption pair with a shortcut added.
The synthetic shortcuts framework allows us to investigate how much the encoder model relies on the added shortcut during training and evaluation, and hence how much of the relevant information is still captured if a shortcut solution is available.
Overall, our \ac{SVL} framework allows us to investigate the shortcut learning problem in a controlled way. 
We focus on \ac{ICR} as an evaluation task because contrastive losses directly optimize for the \ac{ICR} evaluation task, which assesses the quality of the learned representations by computing a similarity score between images and captions~\citep{radford2021learning, yuksekgonul2023when}.
To investigate the problem, we run experiments on two distinct models:
\begin{enumerate*}[label=(\roman*)]
	\item CLIP~\citep{radford2019language}, a large-scale model that we fine-tune; and
	\item VSE++~\citep{faghri2018improving}, a relatively small model that we train from scratch.
\end{enumerate*}
We evaluate the models' performance on the \ac{Flickr30k}~\citep{young2014image} and \ac{MS-COCO}~\citep{lin2014microsoft, chen2015microsoft} and benchmarks. 
The benchmarks are constructed in such a way that each image is associated with five captions and each caption represents a concise summary of the corresponding image.

Therefore, the contributions of this work are two-fold: 
\begin{enumerate}[label=\Roman*]
	\item \textbf{A framework for investigating the problem of shortcut learning for contrastive \acl{VL} representation learning in a controlled way}:
	We introduce the \textit{\acl{SVL}} framework.
	The framework enables the injection of synthetic shortcuts into image-caption tuples in the training dataset. 
	We use the framework to investigate and understand the extent to which contrastive \ac{VL} models rely on shortcuts when a shortcut solution is available. 
	We run our experiments using CLIP and VSE++, two distinct \acp{VLM}. We evaluate the models' performance on the \ac{Flickr30k} and \ac{MS-COCO} benchmarks.
	We evaluate the effectiveness of contrastive \ac{VL} models by comparing their performance with and without synthetic shortcuts.
	We demonstrate that both models trained from scratch and fine-tuned, large-scale pre-trained foundation models mainly rely on shortcut features and do not learn task-optimal representations.
	Consequently, we show that contrastive losses mainly capture the easy-to-learn discriminatory features that are shared among the image and all matching captions, while suppressing other task-relevant information. 
	Hence, we argue that contrastive losses are not sufficient to learn task-optimal representations for \ac{VL} representation learning.
	\item \textbf{We present two shortcut learning reduction methods on our proposed training and evaluation framework:}  We investigate \acf{LTD} and \acf{IFM} using our \ac{SVL} training and evaluation framework. 
	While both methods improve performance on the evaluation task, our framework poses challenges that existing shortcut reduction techniques can only partially address, as the performance is not on par with models trained without synthetic shortcuts.
	These findings underline the importance and complexity of our framework in studying and evaluating shortcut learning within the context of contrastive VL representation learning.
\end{enumerate}

%% file: figures/latent_viz/latent_viz.tex

\begin{figure*}[t!]
	\centering
	\begin{subfigure}[t]{0.37\textwidth}
		\centering
		\includegraphics[width=0.75\textwidth]{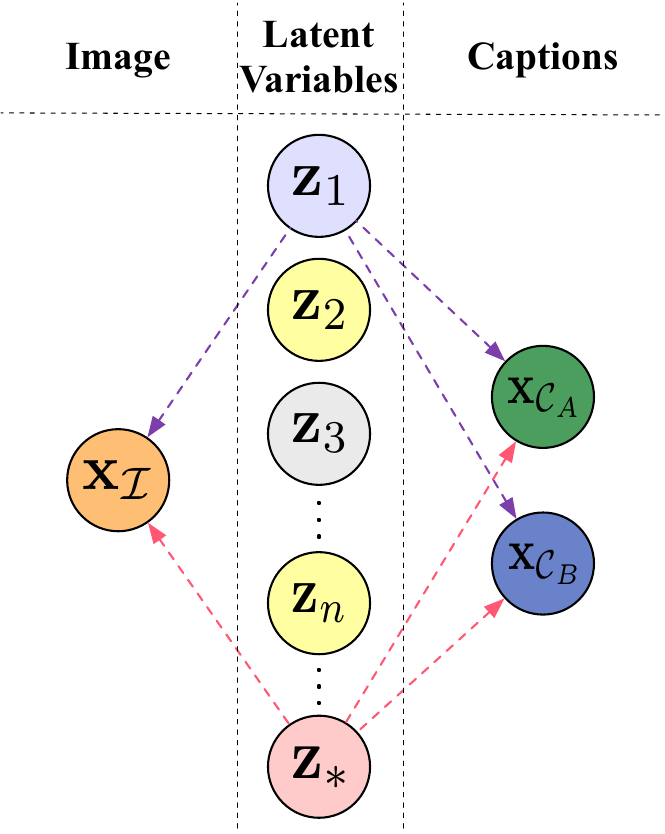}
		\caption{Minimal shared information.}
	\end{subfigure}%
	\qquad
	\begin{subfigure}[t]{0.37\textwidth}
		\centering
		\includegraphics[width=0.75\textwidth]{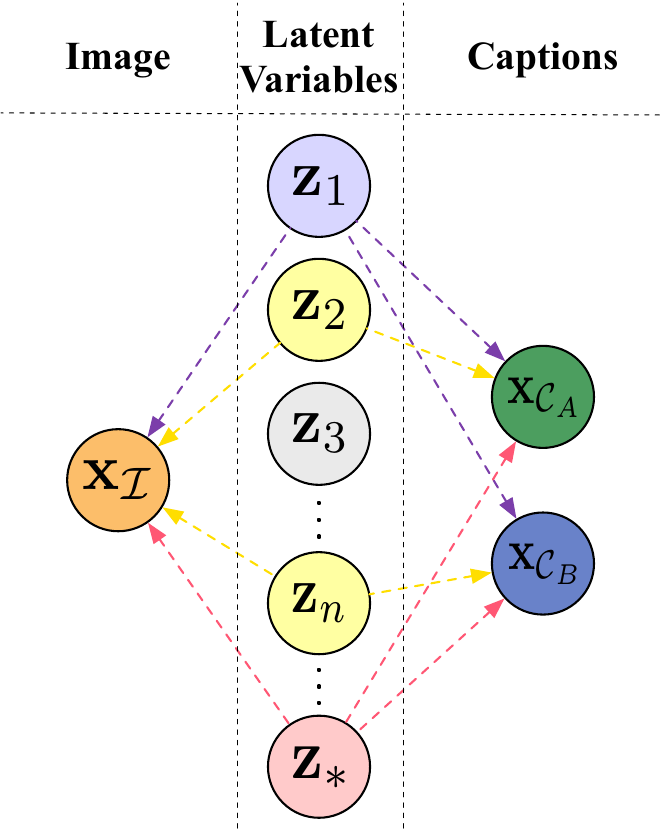}
		\caption{Task-optimal information.}
	\end{subfigure}
	\caption{
	Synthetic shortcuts in the context of minimal shared and task-optimal information for vision-language representation learning with multiple captions per image. 
        The purple color represents features shared among the image and all captions (minimal shared information).
        The yellow color represents caption-specific features (unique information).
        The grey color indicates features that are not present in both the image and any of the captions (task-irrelevant information).
        The red color indicates synthetic shortcuts. 
        We demonstrate that while shortcuts exist in both scenarios, minimal shared information also includes information shared among the image and all associated captions, whereas task-optimal information combines both minimal shared information and caption-specific information.
	}
	\label{fig:latent_viz}
\end{figure*} 

%% file: sections/02_background.tex

\section{Background and Analysis}
\label{sec:background}

In this section, we present the notation, setup, and assumptions on which we base the work. Additionally, we conduct an analysis of contrastive \ac{VL} representation learning with multiple captions per image.

\subsection{Preliminaries}
\label{subsec:notation-assumptions}

\header{Notation} We closely follow the notation from \citep{bleeker2023reducing}. 
See Table~\ref{tab:notation} for an overview.
Let $\mathcal{D}$ be a dataset of $N$ image-caption tuples: $\mathcal{D} = \left\{\left(\img{i}, \{\capt{i}{j}\}_{j=1}^{k} \right)\right\}_{i=1}^{N}$. 
Each tuple $i \in N$ contains one image $\img{i}$ and $k$ captions $\capt{i}{j}$, where $1 \leq j \leq k$. 
All captions in tuple $i \in N$ are considered as matching captions w.r.t.\ image $\img{}$ in the tuple $i$.
The latent representation of an image-caption pair from a tuple $i$ is denoted as $\zimg{i}{}$ and $\zcapt{i}{j}$ respectively.
During training, we sample image-caption pairs from the dataset $\mathcal{D}$ and optimize for the evaluation task $T$. 
We include all captions in the dataset once per training epoch, hence, each image is sampled $k$ times.

Given an image $\img{}$, a set of $k$ associated captions $K = \{\capt{}{j}\}_{j=1}^{k}$, and one caption randomly sampled from the set $\capt{}{} \in K$,
we define the following representations:
\begin{enumerate*}[label=(\roman*)]
	\item $\zcaptsuff{}{}$ as \emph{sufficient} representation of the caption $\capt{}{}$ that describes the image $\img{}$;
	\item $\zimgsuff{}$ as representation of the image $\img{}$ \emph{sufficient for the caption} $\capt{}{}$;
	\item $\zminsuff{}$ as representation of the image $\img{}$ that is \emph{minimally sufficient for the caption} $\capt{}{}$; and
	\item $\zoptimal{}{K}$ as representation of the image $\img{}$ that is \emph{optimal for the set of captions} $K$ given the task $T$.
\end{enumerate*}

In addition, we write
$\sct{}$ for a synthetic shortcut,
$S$ for the original shared information, i.e., information that does not contain synthetic shortcuts,
$S^{+}$ for the shared information that includes a synthetic shortcut,
and $R^{+}$ for task-relevant information that contains a synthetic shortcut.

In the context of task relevance, we define $R$ and $\neg R$ as task-relevant and task-irrelevant information, respectively, and $C$ as task-relevant information specific for caption $\capt{}{}$.

\header{Setup} 
We work with a dual-encoder setup, with an image encoder and a caption encoder that do not share parameters. 
The \emph{image encoder} $f_{\theta}(\cdot)$ takes an image $\img{}$ as input and returns its latent representation: $\zimg{}{} :=  f_{\theta}(\img{})$.
Similarly, the \emph{caption encoder} $g_{\phi}(\cdot)$ takes a caption $\capt{}{}$ as input, and encodes the caption into a latent representation: $\zcapt{}{} :=  g_{\phi}(\zcapt{}{})$. 
Both $\zcapt{}{}$ and $\zimg{}{}$ are unit vectors projected into $d$-dimensional multi-modal space: $\zcapt{}{} \in \sR^d$, $\zimg{}{} \in \sR^d$.
For an overview of notation, we refer to Appendix~\ref{appendix:notation}, Table~\ref{tab:notation}.

\header{Assumptions}
Given an image-caption tuple, 
we assume that each caption in the tuple is distinct from the other captions in the tuple.
We also assume that each caption in the tuple contains two types of task-relevant information:
\begin{enumerate*}[label=(\roman*)]
	\item shared information, i.e., information shared with other captions in the same tuple,
	and 
	\item caption-specific information, i.e., information that is not shared with the other captions.
\end{enumerate*}
For simplicity, we base our subsequent analysis on tuples where one image $\img{}{}$ is associated with two captions $\capt{}{A}$ and $\capt{}{B}$: $\left(\img{}, \{\capt{}{A}, \capt{}{B}\} \right)$. 
However, the analysis described in this section can be extended to a case with more than two captions. We treat images and captions as views and define  $\img{}{}$, $\capt{}{A}$, and $\capt{}{B}$ to be random variables of an image and two matching captions, with the joint distribution $p(\img{}{}, \capt{}{A}, \capt{}{B})$.
For more details on assumptions and problem definition, we refer to Appendix~\ref{appendix:problem}.

\subsection{Analysis of Contrastive Vision-Language Representation Learning for Multiple Captions per Image}

\header{InfoMax} We start our analysis of contrastive \ac{VL} representation learning by introducing the InfoMax optimization objective, a typical loss for \ac{VL} representation learning.
The goal of an InfoMax optimization objective, e.g., InfoNCE~\citep{oord2018representation}, is to maximize the \ac{MI} between the latent representations of two views of the same data~\citep{tschannen2020on}. 
Therefore, the optimization objective is equivalent to: $ \max_{f_{\theta}, g_{\phi} } I (\zimg{}{} ; \zcapt{}{} ) $ where $\zimg{} := f_{\theta}(\img{}{})$ and $\zcapt{}{} := g_{\phi}(\capt{}{})$.

\header{Minimally Sufficient Image Representation} During training, batches of image-caption pairs are sampled. 
The optimization involves maximizing the \ac{MI} between the image representation $\zimg{}{}$ and the matching caption representation $\zcapt{}{}$.
\cite{wang2022rethinking} argue that, since all supervision information for one view (i.e., the image) comes from the other view (i.e., the caption), the representations learned contrastively are approximately minimally sufficient.
Following \citep{tian2020what, wang2022rethinking}, we extend the definition of sufficient representation to \ac{VL} context and define sufficient caption representations, sufficient image representations, and minimally sufficient image representation.

\begin{definition}[Sufficient caption representation]
	\label{def:suff_capt_repr}	
	Given an image $\img{}{}$, and a set of matching captions $\mathcal{C} = \{ \capt{}{A}, \capt{}{B} \}$,
	the representation $\zcaptsuff{}{}$ of caption $\capt{}{} \in \mathcal{C}$ is sufficient for image $\img{}{}$ if, and only if, $I(\zcaptsuff{}{}; \img{}{} ) = I(\capt{}{}; \img{}{})$.
\end{definition}
The sufficient caption representation $\zcaptsuff{}{}$ contains all the information about image $\img{}{}$ in caption $\capt{}{}$.
 
\begin{definition}[Sufficient image representation]
	\label{def:suff_img_repr}		
	Given an image $\img{}{}$, and a set of matching captions $\mathcal{C} = \{ \capt{}{A}, \capt{}{B} \}$,
	the representation $\zimgsuff{}$ of image $\img{}{}$ is sufficient for caption $\capt{}{} \in \mathcal{C}$ if, and only if,  $I(\zimgsuff{}; \capt{}{} ) = I(\img{}{} ; \capt{}{})$.
\end{definition}

Similarly, the sufficient image representation $\zimgsuff{}$ contains all the shared information between an image $\img{}{}$ and a caption $\capt{}{}$. 
Note that a sufficient image representation can be sufficient w.r.t.\ multiple captions.

\begin{definition}[Minimally sufficient image representation]
	\label{def:zimgminsuff}		
	Given an image $\img{}{}$, and a set of matching captions $\mathcal{C} = \{ \capt{}{A}, \capt{}{B} \}$,
	the sufficient image representation $\zminsuff{}$ of image $\img{}{}$  is minimally sufficient for caption $\capt{}{} \in \mathcal{C}$ if, and only if, $I(\zminsuff{}; \img{}{}) \leq I(\zimgsuff{}; \img{}{})$, for all $\zimgsuff{}$ that are sufficient.
\end{definition}
Intuitively, $\zminsuff{}$ comprises the smallest amount of information about $\img{}{}$ (while still being sufficient) and, therefore, only contains the information that is shared with caption $\capt{}{}$, i.e., the non-shared information is suppressed.

\begin{figure*}[t!]
	\centering
	\includegraphics[width=0.85\textwidth]{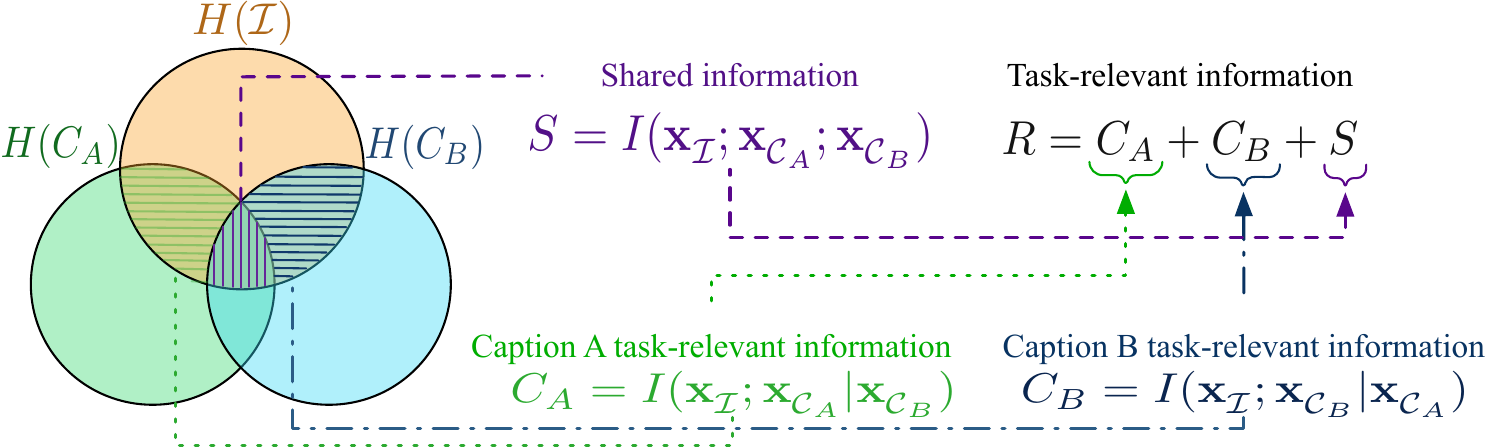}
	\caption{
		We define $H(\img{}{})$ as image information, $H(\capt{}{A})$ and $H(\capt{}{B})$ as caption information;
		both captions only describe the information depicted in the image and contain shared and caption-specific information.
		We further define
		$C_A = I( \img{}{} ; \capt{}{A} \mid \capt{}{B})$
		and
		$C_B = I( \img{}{} ; \capt{}{B} \mid \capt{}{A})$ as caption-specific information; 
		$S = I(\img{}{} ; \capt{}{A} ; \capt{}{B})$ as shared information;
		$\neg R = H( \img{}{} \mid \capt{}{A} , \capt{}{B} ) $ as task-irrelevant information;
		$R = C_A + C_B + S$ as task-relevant information.
	}
	\label{fig:venn}
\end{figure*} 

\header{Task-Optimal Image Representation} The definition of task-optimal image representation is based on the notion of task-relevant information. 
In the context of \ac{VL} representation learning with multiple captions per image, we define task-relevant information as all information described by the matching captions. That includes both caption-specific and shared information.
Consequently, task-optimal image representation is image representation that is sufficient w.r.t. all matching captions.

Formally, following assumptions from Appendix~\ref{app:assumptions}, we define task-relevant information $R$ as all the information described by the matching captions. The task-relevant information can be expressed as follows:
\begin{equation}
\begin{split}
	\underbrace{R}_{\substack{\text{Task-relevant}\\ \text{information}}} & =
	\underbrace{H(\img{}{} )}_{\substack{\text{Image}\\ \text{information}}}
	- \underbrace{H(\img{}{} \mid \capt{}{A} , \capt{}{B})}_{\substack{\text{Task-irrelevant}\\ \text{information}}}
	\\
	& = 
	\underbrace{I( \img{}{} ; \capt{}{A} \mid \capt{}{B})}_{\substack{\text{$C_{A}$-specific}\\ \text{task-relevant information}}}
	+
	\underbrace{I( \img{}{} ; \capt{}{B} \mid \capt{}{A})}_{\substack{\text{$C_{B}$-specific}\\ \text{task-relevant information}}}
	+
	\underbrace{I(\img{}{} ; \capt{}{A};\capt{}{B})}_{\substack{\text{Shared}\\ \text{information}}}.
\end{split}	
	\label{eq:task_relevant}
\end{equation}
Similarly, task-irrelevant information $\neg R$ is the image information not described by the captions. Figure~\ref{fig:venn} illustrates both definitions.

The multi-view assumption states that task-relevant information for downstream tasks comes from the information shared between views~\citep{shwartz2023compress}.  
However, in the case of \ac{VL} representation learning with multiple captions per image, task-relevant information $R$ includes both shared information $S$, and caption-specific information $C_A$ and $C_B$ (Eq.~\ref{eq:task_relevant}).

\begin{definition}[Task-optimal image representation]
	\label{def:opt_z}	
	Given an image $\img{}{}$, and a set of matching captions $\mathcal{C} = \{ \capt{}{A}, \capt{}{B} \}$, 
	the representation $\zoptimal{}{\mathcal{C}}$ is task-optimal image representation for all matching captions if, and only if, $I(\zoptimal{}{\mathcal{C}}; \capt{}{}) = I(\img{}{}; \capt{}{})$, for all $\capt{}{} \in \mathcal{C}$.
\end{definition}
In other words, task-optimal image representations contain all the information that the image shares with the matching captions. 
Hence, a task-optimal image representation is sufficient w.r.t.\ all matching captions. 
The information contained in the task-optimal image representation includes both shared and caption-specific information.
Therefore, a task-optimal image representation can never be a minimally sufficient image representation w.r.t.\ to a specific caption.

\begin{thm}{1}[Suboptimality of contrastive learning with multiple captions per image]
	\label{thm:suboptimality-main}	
	Given an image $\img{}{}$, a set of matching captions $\mathcal{C} = \{ \capt{}{A}, \capt{}{B} \}$, and a contrastive learning loss function $\infonce{}$ that optimizes for task $T$, 
	image representations learned during contrastive learning will be minimally sufficient and will never be task-optimal image representations.
\end{thm}

The proof is provided in Appendix~\ref{app:analysis-of-contrastive}.
Rephrasing Theorem~\ref{thm:suboptimality-main}, given an image and two captions that form two image-caption pairs, $(\img{}{}, \capt{}{A})$ and $(\img{}{}, \capt{}{B})$, and assuming that contrastive loss optimizes the image encoder to be minimally sufficient w.r.t.\ to caption $\capt{}{A}$ during a training step, all task-relevant information $C_{B}$ specific to caption $\capt{}{B}$ will be suppressed in $\zimg{}{}$. Hence, the resulting image representation will not be optimal for the task $T$.

Theorem~\ref{thm:suboptimality-main} indicates a discrepancy between minimally sufficient representations learned during contrastive training with the InfoNCE loss and the task-optimal image representations in the context of learning \ac{VL} representations with multiple captions per image.
Although the InfoMax loss does not have an explicit constraint to compress information, prior work indicates that feature suppression is happening \citep{shwartz2023compress, robinson2021can}. 
Hence, we question if contrastive loss can be used to learn task-optimal image representations in the context of multiple captions per image.

Furthermore, Theorem~\ref{thm:suboptimality-main} implies that in the context of contrastive \ac{VL} representation learning with multiple captions per image, the minimally sufficient representation, which discards non-shared information, is not the same as the task-optimal representation that comprises both caption-specific and shared information.  
This suggests that the features learned during contrastive learning might be shortcuts, i.e., easy-to-detect discriminatory features that minimize the contrastive optimization objective but are not necessarily sufficient for solving the evaluation task.
To examine this problem, we introduce a synthetic shortcuts framework that allows us to investigate the problem of suboptimality of contrastive learning with multiple captions per image in a controlled way.

%% file: sections/03_shortcuts_setup.tex

\section{Synthetic Shortcuts to Control Shared Information}
\label{eval:method}

In Section~\ref{sec:background} we show the suboptimality of the contrastive InfoNCE loss with multiple captions per image.
In the case of real-world \ac{VL} datasets with multiple captions per image, there are no annotations that indicate the information shared between the image and captions and the information specific to each caption.
Hence, we cannot directly measure how much of the shared and unique information is captured by the representations.

\begin{wrapfigure}[17]{R}{0.4\textwidth}
	\centering
	\includegraphics[width=0.4\textwidth]{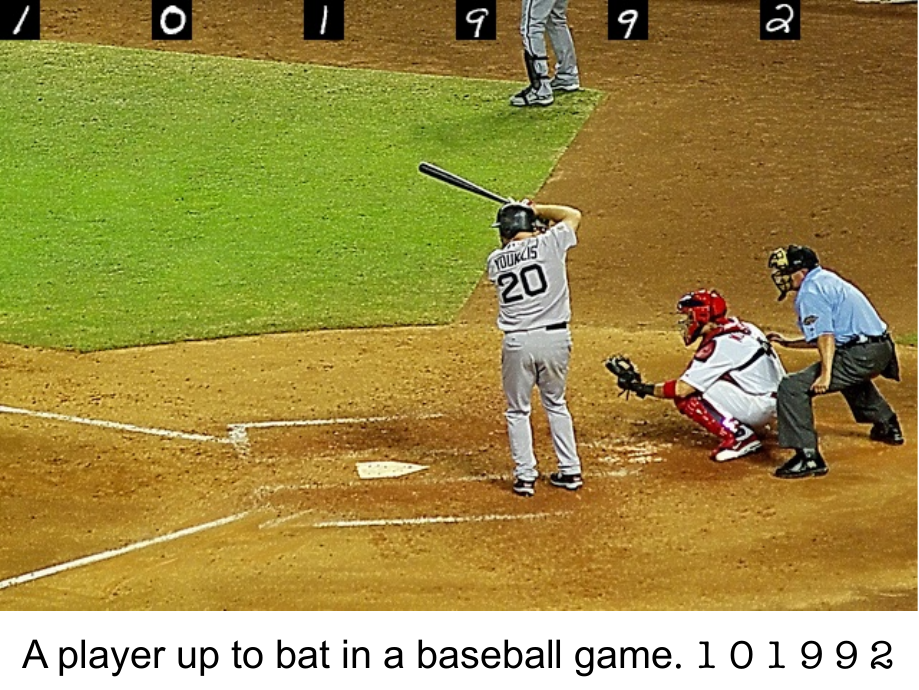}
	\caption{An image-caption pair from the MS-COCO dataset with a shortcut added to both the image and the caption.}
	\label{fig:shortcut-example}
\end{wrapfigure}

\header{Synthetic Shortcuts}
In this section, we introduce the \textit{\acf{SVL}} training and evaluation framework.
We denote the \textit{synthetic shortcuts for image-caption data} as $\sct{}$.
The purpose of the framework is to introduce additional and easily identifiable information shared between an image and the matching captions that lacks any semantic meaning. 
The shortcuts we use in this work are represented as numbers that we add to images and captions.
For images, we add the shortcut number by adding MNIST images as an overlay to the original images.
For captions, we append the numbers of the shortcut as extra tokens at the end of the caption.

Figure~\ref{fig:shortcut-example} illustrates an example of an image-caption pair with an added shortcut.
The example contains an image with the caption: `A player up to bat in a baseball game. 1 0 1 9 9 2.'
Here, `1 0 1 9 9 2' is a shortcut added to both the image and the caption. 
For the image modality, we add the shortcut by overlaying MNIST images at the top of the original image.
For the text modality, we append the shortcut as additional tokens at the end of the caption.
This identifier provides an additional link between the image and the caption without carrying any semantic meaning related to their content. 
Additional examples are shown in Figure~\ref{fig:shortcut_examples}.

If contrastive losses learn task-optimal representations, then the presence of synthetic shortcuts should not negatively impact the evaluation performance, since synthetic shortcuts represent additional information and the remaining task-relevant information is intact.
By incorporating synthetic shortcuts into the image-caption dataset, the shared information would include the information that was originally shared and the synthetic shortcut:
$S^{+} = S + \sct{}$.
Hence, the task-relevant information would comprise caption-specific information that was originally shared and a synthetic shortcut:
$R^{+} = C_A + C_B + S + \sct{}$. 
If injecting a synthetic shortcut influences the performance negatively, we can conclude that by learning to represent a synthetic shortcut the model suppresses other task-relevant information in favor of the shortcut, hence the representation is not task-optimal.
The setup is inspired by the ``datasets with explicit and controllable competing features,'' introduced by \cite{chen2021intriguing}, but we adapt this setup to the \ac{VL} scenario.

For experiments, we use the \ac{Flickr30k} and \ac{MS-COCO} image-caption datasets, that consist of image-caption tuples, each image is associated with five captions.
During training, we sample a batch of image-caption pairs $\mathcal{B} = \{(\img{i}, \capt{i}{j}), \dots\}_{i=1}^{|\mathcal{B}|}$, from dataset $\mathcal{D}$, and apply shortcut sampling.
We inject the shortcuts in a manner that preserves the original information of the images and captions. 
Furthermore, we append the shortcut after applying data augmentations to ensure that the shortcut is present in both the images and captions (i.e., the shortcut is not augmented away).
We refer to Figure~\ref{fig:shortcut_examples} for some examples.
The training, evaluation, and implementation details of the shortcut sampling are provided in Appendix~\ref{app:experimental-shortcutsampling}.

We define the following experimental setups:
\begin{enumerate}[label=\Roman*]
	\item \emph{No shortcuts}: As a baseline, we fine-tune a pre-trained CLIP~\citep{radford2021learning} and train VSE++~\citep{faghri2018improving} from scratch on \ac{Flickr30k} and \ac{MS-COCO}, without using any shortcuts. 
	The experimental setup for training both models is provided in Appendix~\ref{app:experimental-models} and \ref{app:experimental-training}. 
	The goal of this setup is to show the retrieval evaluation performance without adding any shortcuts for both a large-scale pre-trained foundation model and a small-scale model trained from scratch.
	\item \emph{Unique shortcuts}: We add a unique shortcut to each image-caption tuple $i \in\mathcal{D}$ in the dataset. 
	In this setup, each image caption pair can be uniquely matched during training by only detecting the shortcut. 
	For each tuple $i \in \mathcal{D}$, we use the number $i$ as the number of the shortcut we inject to the image and captions in the tuple.
	If the contrastive loss learns task-optimal representations, the downstream evaluation performance should not decrease when training with unique shortcuts.
	\item \emph{Unique shortcuts on only one modality}: To show that the shortcuts do not interfere with the original task-relevant information ($S, C_A$, and $C_B$) of the images and captions, we create a dataset with only shortcuts on either the image or caption modality. Therefore, the shortcut cannot be used by the encoders to match an image-caption pair. 
	Hence, we expect the encoders to ignore the shortcuts and extract the features from the original data similar to the features learned by the baseline models in experimental setup \rom{1}.
	\item \emph{N bits of shortcuts}: In this setup, for each image-caption pair in the training batch $\mathcal{B}$, we randomly sample a shortcut number from the range $[0, 2^{n}]$, where $n$ is the number of bits.
	The higher the value of $n$, the more image-caption pairs in the training batch will have by expectation a unique shortcut, and, the less the model has to rely on $S$ and the remaining task-relevant information to solve the contrastive objective. 
	The goal of this setup is to show that, the more unique (shortcut) information is present per sample in the batch, the less contrastive models rely on the remaining task-relevant information.
\end{enumerate}
 It should be noted that the shortcuts we add are independent of the image-caption pairs. 
 However, the goal of the \ac{SVL} framework is to measure the effect of the presence of additional easy-to-detect shared information on the learned representations. 

 \header{Evaluation Method}
To show the effect of the injected shortcuts on retrieval evaluation performance, we evaluate both with and without adding the shortcuts during evaluation.
When training with unique shortcuts, we add a unique shortcut to each tuple in the test set as well.
When training with shortcuts on either one of the two modalities, we only evaluate without shortcuts to show that training with shortcuts on one modality does not influence performance.
When training with $n$ bits of shortcuts, we add the shortcut $\mod(i, n)$ (modulo) to each tuple $i$ in the evaluation set, to make sure we use the same number of shortcuts during evaluation as during training.
To facilitate the reproducibility and support further research, we provide the code with our paper.\footnote{ \url{https://github.com/MauritsBleeker/svl-framework}}

%% file: sections/04_shortcuts_results.tex
\section{Synthetic Shortcuts and their Impact on Learned Representations and Evaluation Performance}
\label{eval:results}

\begin{figure}[t!]
\centering
\begin{subfigure}[b]{\textwidth}
	\includegraphics[width=1\linewidth]{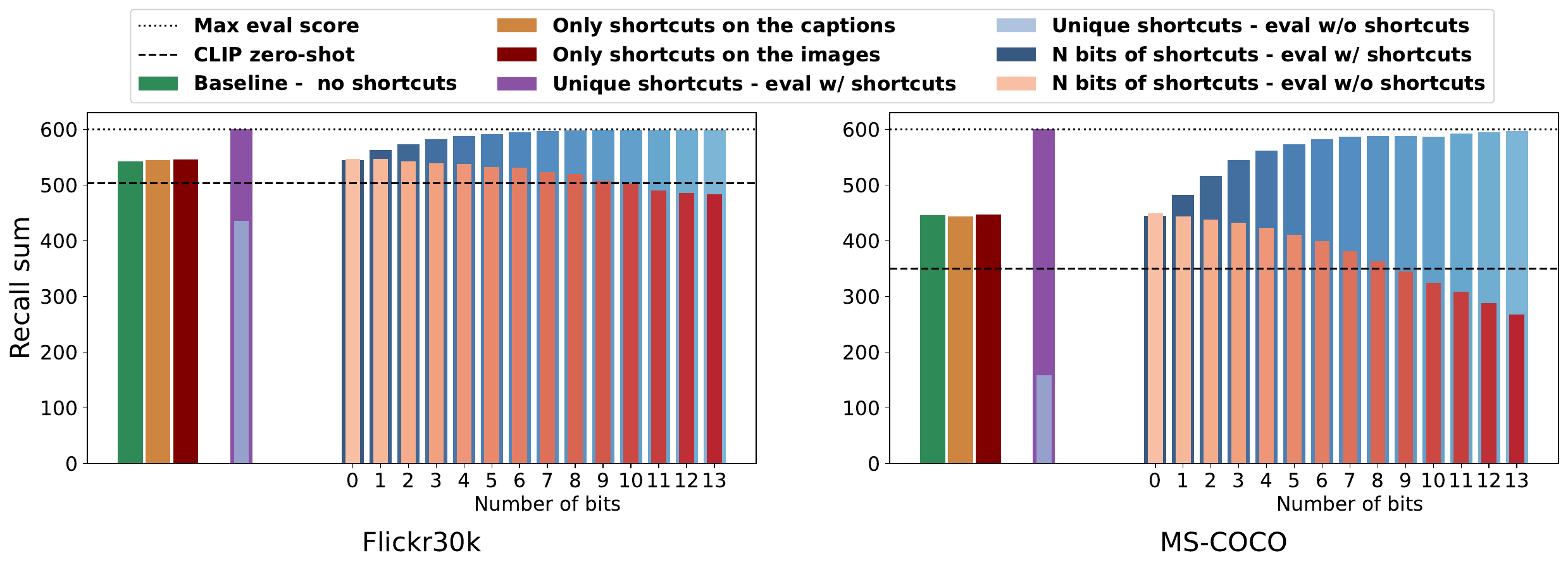}
	\vspace*{-6mm}
	\caption{Evaluation results for the CLIP model when using different shortcut sampling setups. }
	\label{fig:results_clip}
\end{subfigure}

\vspace*{4mm}
\begin{subfigure}[b]{\textwidth}
	\includegraphics[width=1\linewidth]{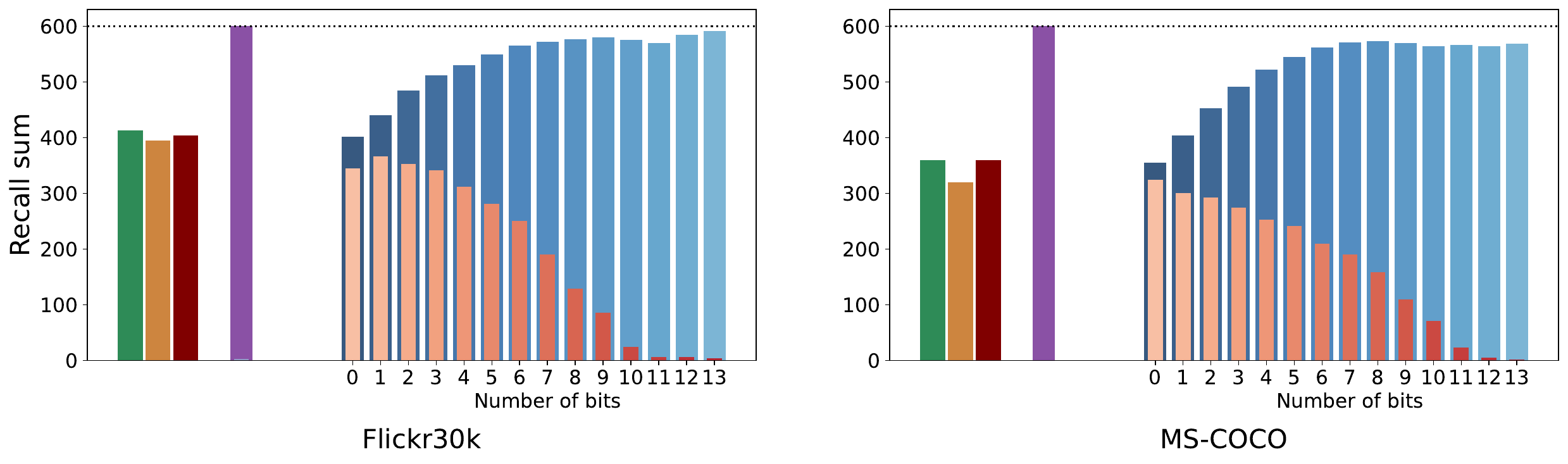}
	\vspace*{-6mm}
	\caption{Evaluation results for the VSE++ model when using different shortcut sampling setups. }
	\label{fig:results_vse} 
\end{subfigure}
\caption{Effect of synthetic shortcuts on CLIP and VSE++ performance on \ac{ICR} task.
The dotted line represents the maximum achievable recall sum, while the dashed line for CLIP indicates its zero-shot evaluation performance
(Best viewed in color.)}
\label{fig:shortcuts_results}
\end{figure}

\subsection{Findings}
First, we train and evaluate both a CLIP and VSE++  without shortcuts on the \ac{Flickr30k} and \ac{MS-COCO} dataset for the image-caption retrieval task as a baseline. 
We use the recall sum (i.e., the sum of $R@1, R@5$, and $R@10$ for both \ac{i2t} and \ac{t2i} retrieval) as evaluation metric (see Appendix~\ref{appendix:task} for the evaluation task description). 
We visualize the results in~Figure~\ref{fig:shortcuts_results}.
The dotted line (in Figure~\ref{fig:results_clip} and~\ref{fig:results_vse}) indicates the maximum evaluation score (i.e., 600). 
For CLIP, we also provide the zero-shot performance of the model, indicated by the dashed line in Figure~\ref{fig:results_clip}. 
When referring to specific results in~Figure~\ref{fig:shortcuts_results}, we use the color of the corresponding bar and legend key in brackets in the text.

Based on Figure~\ref{fig:shortcuts_results}, we draw the following conclusions: 
\begin{enumerate}[label=\Roman*]
	\item When training CLIP and VSE++ with only shortcuts on either the caption modality (in Figure~\ref{fig:shortcuts_results}, the corresponding bar/legend box is colored \caponly) or on the image modality~(\imgonly, in Figure~\ref{fig:shortcuts_results}), we do not observe a drop in evaluation scores for CLIP compared to the baseline model (\baseline, in Figure~\ref{fig:results_clip}).
	For VSE++ we only observe a slight drop in evaluation score when training with shortcuts on the caption modality (again \caponly, mainly for \ac{MS-COCO}, in Figure~\ref{fig:results_vse}).  
	Therefore, we conclude that the synthetic shortcuts do not interfere with the original shared information $S$ or other task-relevant information.
	
	\item When training the models with \textit{unique shortcuts}, we observe for both CLIP and VSE++ that when evaluating with shortcuts~(\uswith, in Figure~\ref{fig:shortcuts_results}), the models obtain a perfect evaluation score. 
	When evaluating without shortcuts (\uswithout, in Figure~\ref{fig:shortcuts_results}) the evaluation score for VSE++ drops to zero and for CLIP below the zero-shot performance. 
	We conclude that with unique shortcuts:
	\begin{enumerate*}[label=(\roman*)]
		\item both CLIP and VSE++ fully rely on the shortcuts to solve the evaluation task,
		\item VSE++ has not learned any other shared or task-relevant information other than the shortcuts (since it is trained from scratch, only detecting the shortcuts is sufficient to minimize the contrastive loss), and 
		\item fine-tuned CLIP has suppressed original features from the zero-shot model in favor of the shortcuts.	
	\end{enumerate*}
	
	\item When training the models with \textit{N bits of shortcuts}, we observe for both CLIP and VSE++ that the larger the number of bits we use during training and when evaluating without shortcuts~(\bswithout, in Figure~\ref{fig:shortcuts_results}), the bigger the drop in evaluation performance. 
	When we evaluate with shortcuts~(\mbox{\bswith,} in Figure~\ref{fig:shortcuts_results}), the evaluation performance improves as we use more bits compared to the baseline without shortcuts~\baseline, in Figure~\ref{fig:shortcuts_results}). 
	For VSE++, evaluating without shortcuts~(\bswithout, in Figure~\ref{fig:results_vse}) results in a drop to zero when having a large number of bits.
	 For CLIP, the evaluation performance drops below the zero-shot performance. 
	 If we train with 0 bits of shortcuts (i.e., the shortcut is a constant) we do not observe any drop or increase in evaluation scores for CLIP.
\end{enumerate}

\subsection{Upshot}
Given the findings based on Figure~\ref{fig:shortcuts_results} we conclude that a contrastive loss (i.e., InfoNCE) mainly learns the easy-to-detect minimal shared features among image-caption pairs that are sufficient to minimize the contrastive objective while suppressing the remaining shared and/or task-relevant information.  
If contrastive losses are sufficient to learn task-optimal representations for image-caption matching, these shortcuts should not adversely impact the evaluation performance.
Moreover, if the contrastive loss would only learn features that are shared among the image and all captions (i.e, $S$), we should not observe a drop in performance to 0 for the VSE++ model when training with unique shortcuts, since there is still a lot of task-relevant information present in $S$.
Especially in a training setup where a model is trained from scratch or fine-tuned on small datasets, the easy-to-detect features are likely not equivalent to all task-relevant information in the images and captions.
Hence, we conclude that contrastive loss itself is not sufficient to learn task-optimal representations of the images (and sufficient representations of captions) and that it only learns the minimal easy-to-detect features that are needed to minimize the contrastive objective.

%% file: sections/05_evaluation_setup.tex

\section{Reducing Shortcut Learning}
\label{sec:reducing_shortcut_solution}

In the earlier section, we have demonstrated that contrastive loss mainly relies on the minimal, easy-to-detect features shared among image-caption pairs while suppressing remaining task-relevant information.
In this section, we describe two methods that help to reduce shortcut learning for contrastive learning on our \ac{SVL} framework: \Acl{LTD}~\citep{bleeker2023reducing} and  \acl{IFM}~\citep{robinson2021can}. 

\subsection{Latent Target Decoding}  

\Acf{LTD}~\citep{bleeker2023reducing} is a method to reduce predictive feature suppression (i.e., shortcut learning) for resource-constrained contrastive image-caption matching. 
The contrastive objective (i.e., InfoNCE)  is combined with an additional reconstruction loss, which reconstructs the input caption from the latent representation of the caption $\zcapt{i}{j}$. 
We refer to Appendix~\ref{app:ltd} for the mathematical definition of \ac{LTD}.
Instead of reconstructing the tokens of the input caption in an auto-regressive manner (i.e., auto-encoding), the caption is reconstructed non-auto-regressively, by mapping the caption representation into the latent space of a \acl{SBERT}~\citep{reimersers2019sentence, song2020mpnet} and minimizing the distance (i.e., reconstructing) between the reconstruction and the \acl{SBERT} representation of the caption $\capt{i}{j}$. 
The assumption is that the \emph{target} generated by the \acl{SBERT} model contains all task-relevant information in the caption. 
Hence, by correctly mapping the latent caption representation $\zcapt{i}{j}$ into the latent space of \acl{SBERT}, the caption encoder cannot suppress any task-relevant information or rely on shortcut solutions. 
\ac{LTD} is implemented both as a dual-loss objective (i.e., the contrastive loss and \ac{LTD} are added up) and as an optimization constraint while minimizing the InfoNCE loss, by implementing the loss as a Lagrange multiplier. For the mathematical definition of \ac{LTD}, we refer to Appendix~\ref{app:ltd}.

\header{Experimental Setup} We use the \ac{LTD} implementation and set-up similar to \citet{bleeker2023reducing}. 
We train both CLIP and VSE++ with \ac{LTD}, implemented as either dual loss or an optimization constraint.
When implementing \ac{LTD} as a constraint, we try $\eta \in \{0.01, 0.05, 0.1, 0.15, 0.2, 0.25, 0.3\}$ as bound values.
Similar to \cite{bleeker2023reducing}, when implementing \ac{LTD} as a dual loss, we use $\beta=1$ as balancing parameters.   
We train both with and without unique shortcuts.  
We do this to show 
\begin{enumerate*}[label=(\roman*)]
	\item what the performance improvement is compared to using only InfoNCE, and 
	\item to what degree \ac{LTD} prevents full collapse to shortcut features.
\end{enumerate*}
For each model and dataset, we take the training setup that results in the highest performance on the validation set.

\subsection{Implicit Feature Modification} 
\Acf{IFM}~\citep{robinson2021can} is a method, originally introduced in the context of representation learning for images,
that applies perturbations to logits used for guiding contrastive models. 
\ac{IFM} perpetuates features that the encoders use during a training step to discriminate between positive and negative samples. 
By doing so, \ac{IFM} alters the features that are currently used to solve the discrimination task, to avoid the InfoNCE loss to learn shortcut solutions.  
How much of the features are removed, is defined by a perturbation budget $\epsilon$.
\ac{IFM} is implemented as a dual loss in combination with the InfoNCE loss. 
For the mathematical definition of \ac{IFM}, we refer to Appendix~\ref{app:ifm}.

\header{Experimental Setup}
We apply a similar experimental set-up for \ac{IFM} as for \ac{LTD}.
We apply \ac{IFM} both to CLIP and to VSE++, both with and without unique shortcuts.
Similar to \citep{robinson2021can}, we try different perturbation budgets $\epsilon$, we try $\epsilon \in \{0.05, 0.1, 0.2, 0.5, 1\}$.
In line with the \ac{LTD} setup, we take the training setup that results in the highest performance on the validation set.

\subsection{Method Comparison}
Both \ac{LTD} and \ac{IFM} aim to mitigate shortcut learning through different approaches.
\ac{LTD} aims to learn all task-relevant information by reconstructing the input captions.
In contrast, \ac{IFM} perturbs the discriminative features in the latent space of the encoder and does not rely on a reconstruction objective.
Overall, both methods represent distinct strategies for improving the robustness and generalization capabilities of \ac{VL} representation learning.

In the following section, we present experimental results with \ac{LTD} and \ac{IFM}, providing insight into their effectiveness in mitigating shortcut learning.

%% file: sections/06_evaluation_results.tex

\section{Experimental Results}
\label{sec:shorcut_results}

\subsection{Does Latent Target Decoding Reduce Shortcut Learning?}

In Table~\ref{tab:ltd} we summarize the effect of \ac{LTD} on reducing shortcut learning. 

For CLIP, for both the \ac{Flickr30k} and \ac{MS-COCO} dataset, we do not observe an increase in recall scores when fine-tuning with $\ltd{}$ compared to models that are only fine-tuned with $\infonce{}$. 
LTD has originally been proposed for resource-constrained \ac{VL} models. 
We argue that the additional features that LTD can extract are either already present in the pre-trained CLIP model, or not relevant for the evaluation task. 
However, when fine-tuning with $\ltd{}$ and in the presence of shortcuts in the training data, degradation in recall scores is significantly lower than when fine-tuned only with the $\infonce{}$. 
This shows that LTD can reduce the suppression of features in favor of the shortcut features when fine-tuning large-scale \ac{VL} models. 

\input{tables/ltd/recall_results_ltd}

Across the board, VSE++ models trained with the $\ltd{}$ loss consistently outperform the $\infonce{}$ loss, both for \ac{i2t} and \ac{t2i} retrieval and both when trained either with or without shortcuts, as indicated by higher recall@$k$ scores; this is consistent with the findings presented in \citep{bleeker2023reducing}).
For both the \ac{Flickr30k} and \ac{MS-COCO} dataset, when trained with the $\infonce{}$ and with shortcuts present in the training data, the model performance collapses to around 0 in the absence of shortcuts (as we have seen in Section~\ref{eval:results}). 
However, when we train with shortcuts in the training data and with $\ltd{}$, we observe, for both \ac{Flickr30k} and \ac{MS-COCO}, a significant gain in performance. 
The performance improvement is bigger for \ac{Flickr30k} than for \ac{MS-COCO}.
In general, the recall scores are still significantly lower than training without shortcuts, however, the models do not solely rely on the shortcuts anymore to minimize the contrastive loss and are able during evaluation (in the absence of shortcuts) to still correctly match image-caption pairs with each other.
The results in Table~\ref{tab:ltd} show that LTD is able, in the presence of shortcuts in the training data, to guide (small-scale) \ac{VL} models that are trained from scratch to not only learn the shortcut features that minimize the contrastive training objective but also represent other remaining task-relevant features in the data that are not extracted by $\infonce{}$.

\subsection{Does Implicit Feature Modification Reduce Shortcut Learning?}

\input{tables/ifm/recall_results_ifm}

In Table~\ref{tab:ifm} we summarize the effect of \ac{IFM} on reducing shortcut solutions. 

For CLIP, we observe that $\ifm{}$, when training without shortcuts in the training data, only improves performance for the \ac{MS-COCO} dataset for the \ac{t2i} task. 
However, for both~\ac{Flickr30k} and~\ac{MS-COCO} we observe that, when training with unique shortcuts in the training data, fine-tuning with $\ifm{}$ results in a significantly lower performance drop in recall score than when fine-tuning with the $\infonce{}$.
Similar to LTD, the recall@$k$ scores are still lower than when trained without shortcuts in the training data.
We conclude that IFM is sufficient to reduce the suppression of features in favor of the shortcut features when fine-tuning a large-scale \ac{VL} model, as indicated by higher recall@$k$ scores when evaluating without shortcuts.

For VSE++, both for the \ac{Flickr30k} and \ac{MS-COCO} dataset, we do not observe that $\ifm{}$ outperforms the $\infonce{}$, both with and without shortcuts present in the training data. 
We even observe that $\ifm{}$, when training without shortcuts, results in a decrease in performance across all recall@$k$ metrics.
When training with $\ifm{}$ and with unique shortcuts in the training data, the evaluation performance still collapses to around 0.
The results in Table~\ref{tab:ifm} show that IFM is not sufficient to prevent models trained from scratch from fully collapsing to the artificial shortcut solutions we introduce in this work (as opposed to LTD).

\subsection{Upshot}

In this section, we have evaluated two methods for reducing shortcut learning on our \ac{SVL} framework: \ac{LTD} and \ac{IFM}.
	\ac{LTD} proves effective in reducing shortcut learning for both CLIP and VSE++.
	\ac{IFM} demonstrates its efficacy solely during the fine-tuning of CLIP.
	These findings indicate that our \ac{SVL} framework is a challenging and interesting framework to study and evaluate shortcut learning for contrastive \ac{VL} models.
	Moreover, our results show that shortcut learning is only partially addressed by the evaluated methods since the evaluation results are not on par with the results on data lacking synthetic shortcuts.

%% file: tables/ltd/recall_results_ltd.tex

\begin{table*}[t!]
	\centering
	\caption{Mean and variance (over three training runs) recall@$k$ evaluation scores for the \ac{Flickr30k} and \ac{MS-COCO} datasets for image-to-text and text-to-image retrieval. 
		We train with two loss functions: $\infonce{}$ and $\ltd{}$. 
		We train either with (\checkmark) or without (\xmark) shortcuts. 
		For the model trained with $\ltd{}$, we provide the hyper-parameters of the best-performing model. $\eta$ indicates that the best-performing model uses LTD implemented as an optimization constraint with bound $\eta$. $\beta$ indicates that the best-performing model uses LTD implemented as a dual-loss with $\beta=1$.
		}
	\setlength{\tabcolsep}{1.5pt}
	\resizebox{\textwidth}{!}{%
	\begin{tabular}{@{\extracolsep{1pt}}l @{} c r@{$_\pm$}l r@{$_\pm$}l r@{$_\pm$}l r@{$_\pm$}l r@{$_\pm$}l r@{$_\pm$}l r@{$_\pm$}l}
		\toprule
		& & \multicolumn{6}{c}{\textit{i2t}} & \multicolumn{6}{c}{\textit{t2i}} & \multicolumn{2}{c}{ \ } \\
		\cmidrule(r){3-8} \cmidrule(r){9-14} \cmidrule{15-16}
		Loss & $\sct{}$ & \multicolumn{2}{c}{R@1}  & \multicolumn{2}{c}{R@5}  & \multicolumn{2}{c}{R@10} & \multicolumn{2}{c}{R@1}  & \multicolumn{2}{c}{R@5}  & \multicolumn{2}{c}{R@10}  & \multicolumn{2}{c}{rsum} \\ 
		\midrule
		\multicolumn{16}{c}{Flickr30k}\\
		\midrule
		& & \multicolumn{14}{c}{CLIP}                                                   \\
		\cmidrule{3-16}
		$\infonce{}$     &   \xmark &   ${86.9}$ & $_{0.1}$ & $\textbf{97.4}$ & $_{0.1}$ & $\textbf{99.0}$ & $_{0.0}$ & ${72.4}$ & $_{0.1}$ & $\textbf{92.1}$ & $_{0.0}$ & $\textbf{95.8}$ & $_{0.0}$ & ${543.5}$ & $_{1.1}$  \\
		$\ltd{}$, $\beta=1$ &  \xmark & ${86.5}$ & $_{0.6}$-- & ${97.1}$ & $_{0.0}\negative$ & ${98.5}$ & $_{0.0}\negative$& ${72.4}$ & $_{0.0}$-- & ${92.3}$ & $_{0.0}\negative$ & ${95.9}$ & $_{0.0}\negative$&  ${542.8}$ & $_{0.8}$--\\
		\cmidrule(r){3-16}
		$\infonce{}$  &  \checkmark  &  ${57.2}$ & $_{8.3}$ & ${84.0}$ & $_{4.8}$ & ${91.0}$ & $_{1.9}$ & ${44.9}$ & $_{4.5}$ & ${74.9}$ & $_{6.0}$ & ${84.2}$ & $_{2.5}$ & ${436.2}$ & $_{145.0}$  \\
		$\ltd{}$,  $\beta=1$ & \checkmark  & $\textbf{64.0}$ & $_{1.3}\positive$ & $\textbf{87.8}$ & $_{0.9}\positive$ & $\textbf{93.2}$ & $_{0.8}\positive$ & $\textbf{50.7}$ & $_{0.6}\positive$ & $\textbf{79.8}$ & $_{0.7}\positive$ & $\textbf{88.1}$ & $_{0.5}\positive$ & $\textbf{463.6}$ & $_{17.3}\positive$ \\
		\cmidrule{3-16}
		&  & \multicolumn{14}{c}{VSE++} \\
		\cmidrule{3-16}
		$\infonce{}$  & \xmark &  ${52.6}$ &$_{1.1}$ & ${79.8}$ & $_{0.1}$ & ${87.8}$ & $_{0.1}$ & ${39.5}$ & $_{0.3}$ & ${69.8}$ & $_{0.0}$ & ${79.4}$ & $_{0.1}$ & ${409.0}$ & $_{4.0}$\\
		$\ltd{}$, $\eta=0.2$ & \xmark & $\textbf{54.1}$ & $_{0.1}\positive$ & $\textbf{81.1}$ & $_{0.8}\positive$ & $\textbf{88.6}$ & $_{0.1}\positive$ & $\textbf{42.5}$ & $_{0.0}\positive$ & $\textbf{71.9}$ & $_{0.1}\positive$ & $\textbf{81.3}$ & $_{0.0}\positive$ & $\textbf{419.6}$ & $_{0.1}\positive$  \\
		\cmidrule(r){3-16}
		$\infonce{}$  &  \checkmark  &  ${0.1}$ & $_{0.0}$ & ${0.6}$ & $_{0.1}$ & ${1.1}$ & $_{0.1}$ & ${0.1}$ & $_{0.0}$ & ${0.5}$ & $_{0.0}$ & ${1.0}$ & $_{0.0}$ & ${3.4}$ & $_{0.6}$ \\
		$\ltd{}$, $\eta=0.05$ & \checkmark  &  $\textbf{24.7}$&$_{0.5}\positive$ & $\textbf{51.8}$ & $_{0.7}\positive$ & $\textbf{65.6}$ & $_{1.4}\positive$ & $\textbf{20.7}$ & $_{1.0}\positive$ & $\textbf{49.2}$ & $_{0.6}\positive$ & $\textbf{62.6}$ & $_{1.2}\positive$ & $\textbf{274.6}$ & $_{4.6}\positive$ \\
		\midrule
		\multicolumn{16}{c}{\ac{MS-COCO}}\\
		\midrule
		& & \multicolumn{14}{c}{CLIP}\\
		\cmidrule(r){3-16} 
		$\infonce{}$     &   \xmark & ${63.8}$ & $_{0.3}$ & ${86.1}$ & $_{0.2}$ & ${92.3}$ & $_{0.0}$ & ${46.3}$ & $_{0.3}$ & ${74.8}$ & $_{0.1}$ & ${84.1}$ & $_{0.2}$ & ${447.5}$ & $_{0.5}$  \\
		$\ltd{}$, $\beta=1$ &  \xmark & ${63.8}$ & $_{0.0}$-- & ${86.1}$ & $_{0.0}$-- & ${92.3}$ & $_{0.0}$-- & ${46.3}$ & $_{0.0}$-- & ${74.7}$ & $_{0.0}$-- & ${84.1}$ & $_{0.0}$-- & ${447.4}$ & $_{0.0}$-- \\
		\cmidrule(r){3-16}
		$\infonce{}$  &  \checkmark & ${13.6}$ & $_{0.9}$ & ${31.5}$ & $_{2.4}$& ${42.2}$ & $_{3.7}$ & ${\phantom{0}7.3}$ & $_{0.6}$ & ${22.1}$ & $_{1.0}$ & ${32.7}$ & $_{1.7}$ & ${149.4}$ & $_{32.7}$  \\
		$\ltd{}$, $\beta=1$ & \checkmark &  $\textbf{18.9}$ & $_{0.1}\positive$ & $\textbf{41.8}$ & $_{0.1}\positive$ & $\textbf{54.1}$ & $_{0.1}\positive$ & $\textbf{16.5}$ & $_{0.0}\positive$ & $\textbf{39.4}$ & $_{0.0}\positive$ & $\textbf{52.6}$ & $_{0.1}\positive$ & $\textbf{223.4}$ & $_{0.2}\positive$ \\
		\cmidrule{3-16} 
		&  & \multicolumn{14}{c}{VSE++} \\
		\cmidrule{3-16} 
		$\infonce{}$  & \xmark  & ${42.2}$ & $_{0.1}$ & ${72.7}$ & $_{0.1}$ & ${83.2}$ & $_{0.1}$ & ${30.9}$ & $_{0.0}$ & ${61.2}$ & $_{0.1}$& ${73.5}$ & $_{0.1}$ & ${363.8}$ & $_{2.3}$ \\
		$\ltd{}$, $\eta=0.1$ & \xmark & $\textbf{43.6}$ & $_{0.1}\positive$ & $\textbf{73.5}$ & $_{0.0}\positive$ & $\textbf{83.7}$ & $_{0.0}\positive$ & $\textbf{32.4}$ & $_{0.1}\positive$ & $\textbf{62.5}$ & $_{0.0}\positive$ & $\textbf{74.7}$ & $_{0.0}$ & $\textbf{370.5}$ & $_{0.1}\positive$   \\
		\cmidrule(r){3-16}
		$\infonce{}$    &  \checkmark & ${0.0}$ & $_{0.0}$ & ${0.1}$ & $_{0.0}$ & ${0.2}$ & $_{0.0}$ & ${\phantom{0}0.0}$ & $_{0.0}$ & ${\phantom{0}0.1}$ & $_{0.0}$ & ${\phantom{0}0.2}$ & $_{0.0}$ & ${\phantom{00}0.7}$ & $_{0.0}$ \\ 
		$\ltd{}$, $\eta=0.01$ & \checkmark & $\textbf{3.9}$ & $_{0.0}\positive$ & $\textbf{13.7}$ & $_{0.6}\positive$ & $\textbf{21.6}$ & $_{0.9}\positive$ & $\textbf{\phantom{0}3.1}$ & $_{0.2}\positive$ & $\textbf{11.0}$ & $_{1.6}\positive$ & $\textbf{18.1}$ & $_{3.0}\positive$ & $\textbf{\phantom{0}71.3}$ & $_{3.6}\positive$  \\
		\bottomrule
	\end{tabular}%
	}
\label{tab:ltd}
\end{table*}

%% file: tables/ifm/recall_results_ifm.tex

\begin{table*}[t!]
	\centering
	\caption{Mean and variance (over three training runs) recall@$k$ evaluation scores for the \ac{Flickr30k} and \ac{MS-COCO} datasets for image-to-text and text-to-image retrieval. 
		We train with two loss functions: $\infonce{}$ and $\ifm{}$. 
		We train either with (\checkmark) or without (\xmark) shortcuts. 
		For the model trained with $\ifm{}$, we provide the hyper-parameters of the best-performing model.}
	\setlength{\tabcolsep}{1.5pt}
	\resizebox{\textwidth}{!}{%
	\begin{tabular}{@{\extracolsep{1pt}}l @{} c r@{$_\pm$}l r@{$_\pm$}l r@{$_\pm$}l r@{$_\pm$}l r@{$_\pm$}l r@{$_\pm$}l r@{$_\pm$}l}
		\toprule
		& & \multicolumn{6}{c}{\textit{i2t}} & \multicolumn{6}{c}{\textit{t2i}} & \multicolumn{2}{c}{ \ } \\
		\cmidrule(r){3-8} \cmidrule(r){9-14} \cmidrule{15-16}
		Loss & $\sct{}$ & \multicolumn{2}{c}{R@1}  & \multicolumn{2}{c}{R@5}  & \multicolumn{2}{c}{R@10} & \multicolumn{2}{c}{R@1}  & \multicolumn{2}{c}{R@5}  & \multicolumn{2}{c}{R@10}  & \multicolumn{2}{c}{rsum} \\ 
		\midrule
		\multicolumn{16}{c}{\ac{Flickr30k}}\\
		\midrule
		& & \multicolumn{14}{c}{CLIP} \\
		\cmidrule(r){3-16} 
		$\infonce{}$ & \xmark & ${86.9}$ & $_{0.1}$ & ${97.4}$ & $_{0.0}$ & ${98.8}$ & $_{0.0}$ & ${72.8}$ & $_{0.2}$ & ${92.1}$ & $_{0.0}$ & ${95.6}$ & $_{0.0}$ & ${543.5}$ & $_{1.3}$\\
		$\ifm{}$, $\epsilon = 0.05$ &  \xmark  & $\textbf{87.4}$ & $_{0.1} \positive$ & ${97.4}$ & $_{0.2}$-- & ${99.1}$ & $_{0.0}$-- & ${73.2}$ & $_{0.0}$-- & ${92.2}$ & $_{0.0}$-- & ${95.6}$ & $_{0.0}$-- & ${544.9}$ & $_{0.2}$-- \\
		\cmidrule(r){3-16}
		$\infonce{}$  &  \checkmark & ${57.9}$ & $_{0.3}$ & ${84.6}$ & $_{0.8}$ & ${91.3}$ & $_{0.0}$ & ${43.9}$ & $_{2.2}$ & ${74.6}$ & $_{0.8}$ & ${84.4}$ & $_{0.4}$ & ${436.7}$ & $_{18.8}$ \\ 
		$\ifm{}$, $\epsilon = 0.1$ & \checkmark & $\textbf{73.8}$ & $_{0.8}\positive$ & $\textbf{91.5}$ & $_{0.5}\positive$ & $\textbf{95.6}$ & $_{0.0}\positive$ & $\textbf{58.9}$ & $_{0.1}\positive$ & $\textbf{84.4}$ & $_{0.1}\positive$ & $\textbf{91.1}$ & $_{0.2}\positive$ & $\textbf{495.2}$ & $_{5.7}\positive$ \\
		\cmidrule{3-16} 
		&  & \multicolumn{14}{c}{VSE++} \\
		\cmidrule(r){3-16} 
		$\infonce{}$  & \xmark & $\textbf{52.9}$ & $_{0.2}$ & $\textbf{80.5}$ & $_{0.1}$ & $\textbf{87.6}$ & $_{0.4}$ & $\textbf{40.5}$ & $_{0.1}$ & ${68.8}$ & $_{0.4}$ & $\textbf{78.9}$ & $_{0.3}$ & $\textbf{409.3}$ & $_{2.6}$ \\
		$\ifm{}$, $\epsilon=0.05$ & \xmark & ${52.4}$ & $_{0.2}\negative$& ${76.9}$ & $_{0.1}\negative$& ${85.3}$ & $_{0.0}\negative$ & ${39.1}$ & $_{0.0}\negative$& ${68.8}$-- & $_{0.1}$ & ${78.2}$ & $_{0.1}\negative$& ${400.7}$ & $_{0.0}\negative$\\
		\cmidrule(r){3-16}
		$\infonce{}$ & \checkmark  & ${0.1}$ & $_{0.0}$ & ${0.4}$ & $_{0.0}$ & ${0.8}$ & $_{0.0}$ & ${0.1}$ & $_{0.0}$ & ${0.4}$ & $_{0.0}$ & ${1.0}$ & $_{0.0}$ & ${2.9}$ & $_{0.0}$ \\
		$\ifm{}$, $\epsilon=0.05$ & \checkmark & ${0.0}$ & $_{0.0}$-- & ${0.6}$ & $_{0.1}$-- & ${0.9}$ & $_{0.2}$-- & ${0.1}$ & $_{0.0}$-- & ${0.5}$ & $_{0.0}$-- & ${1.0}$ & $_{0.0}$-- & ${3.2}$ & $_{0.8}$-- \\
		\midrule
		\multicolumn{16}{c}{\ac{MS-COCO}}\\
		\midrule
		&  & \multicolumn{14}{c}{CLIP} \\
		\cmidrule(r){3-16} 
		$\infonce{}$ & \xmark & $\textbf{63.5}$ & $_{0.1}$ & $\textbf{86.0}$ & $_{0.3}$ & $\textbf{92.2}$ & $_{0.0}$ & ${46.3}$ & $_{0.0}$ & ${74.7}$ & $_{0.0}$ & ${84.2}$ & $_{0.0}$ & ${446.9}$ & $_{0.9}$ \\
		$\ifm{}$, $\epsilon=0.05$ &  \xmark & ${63.0}$ & $_{0.1}\negative$ & ${86.6}$ & $_{0.1}\negative$ & ${92.6}$ & $_{0.2}\negative$ & $\textbf{47.2}$ & $_{0.0}\positive$ & $\textbf{75.6}$ & $_{0.0}\positive$ & $\textbf{84.5}$ & $_{0.0}\positive$ & $\textbf{449.5}$ & $_{1.7}\positive$ \\
		\cmidrule(r){3-16}
		$\infonce{}$  &  \checkmark & ${13.9}$ & $_{0.0}$ & ${32.7}$ & $_{0.1}$ & ${43.8}$ & $_{0.0}$ & ${8.8}$ & $_{0.0}$ & ${24.7}$ & $_{0.2}$& ${35.5}$ & $_{0.5}$ & ${159.4}$ & $_{3.4}$ \\
		$\ifm{}$, $\epsilon=0.05$ & \checkmark & $\textbf{23.4}$ & $_{1.5}\positive$ & $\textbf{46.5}$ & $_{2.7}\positive$ & $\textbf{58.2}$ & $_{2.5}\positive$ & $\textbf{17.1}$ & $_{0.3}\positive$ & $\textbf{38.9}$ & $_{0.9}\positive$ & $\textbf{51.3}$ & $_{1.0}\positive$ & $\textbf{235.5}$ & $_{43.8}\positive$ \\
		\cmidrule{3-16} 
		&  & \multicolumn{14}{c}{VSE++} \\
		\cmidrule(r){3-16}
		$\infonce{}$  & \xmark & $\textbf{41.7}$ & $_{0.3}$ & $\textbf{72.5}$ & $_{0.1}$ & $\textbf{83.1}$ & $_{0.1}$ & $\textbf{31.3}$ & $_{0.0}$ & ${61.1}$ & $_{0.0}$ & ${73.6}$ & $_{0.0}$ & $\textbf{363.4}$ & $_{0.4}$\\
		$\ifm{}$, $\epsilon=0.05$ & \xmark & ${40.2}$ & $_{0.0}\negative$  & ${70.8}$ & $_{0.1} \negative$ & $ {81.6}$ & $_{0.1}\negative$ & ${30.8}$ & $_{0.0}\negative$ & $\textbf{61.5}$ & $_{0.0}\positive$ & $\textbf{74.3}$ & $_{0.0}\positive$ & ${359.3}$ & $_{1.1}\negative$ \\
		\cmidrule(r){3-16}
		$\infonce{}$ &  \checkmark & ${0.0}$ & $_{0.0}$ & ${0.1}$ & $_{0.0}$ & ${0.2}$ & $_{0.0}$ & ${0.0}$ & $_{0.0}$ & ${0.1}$ & $_{0.0}$ & ${0.2}$ & $_{0.0}$ & ${0.6}$ & $_{0.0}$ \\
		$\ifm{}$, $\epsilon=0.05$& \checkmark & ${0.0}$ & $_{0.0}$-- & ${0.1}$ & $_{0.0}$-- & ${0.2}$ & $_{0.0}$ --& ${0.0}$ & $_{0.0}$-- & ${0.1}$ & $_{0.0}$-- & ${0.2}$ & $_{0.0}$-- & ${0.7}$ & $_{0.0}$--  \\
		\bottomrule
	\end{tabular}%
	}
	\label{tab:ifm}
\end{table*}

%% file: sections/07_related_work.tex

\section{Related work}
\label{sec:related-work}

We discuss related work on multi-view representation learning, \acl{VL} learning, and shortcut learning. 

\header{Multi-view Representation Learning}  To learn the underlying semantics of the training data, a subgroup of representation learning methods involves training neural encoders that maximize the agreement between representations of the similar \textit{views}~\citep{oord2018representation, hjelm2019learning, chen2020simple, radford2021learning, bardes2022vicreg}. 
In general, for uni-modal representation learning, data augmentations are used to generate different views of the same data point.
One of the core assumptions in multi-view representation learning is that each view shares the same \emph{task-relevant information}~\citep{sridharan2008information, zhao2017multi, federici2020learning, tian2020contrastive, shwartz2023compress}.
However, the optimal view for contrastive \ac{SSL} (i.e., which information is shared among views/which data augmentation is used) is task-dependent~\citep{tian2020what, xiao2021what}.
Therefore, maximizing the \acf{MI} between representations of views (i.e., shared information) does not necessarily result in representations that generalize better to down-stream evaluation tasks, since the representations may contain too much additional noise that is irrelevant for the downstream task~\citep{tian2020what, tschannen2020on}.
An open problem in multi-view \ac{SSL} is to learn representations that contain all task-relevant information from views where each view contains distinct, task-relevant information~\citep{shwartz2023compress}, this is especially a problem in the multi-modal learning domain~\citep{zong2023self}.

\cite{chen2021intriguing} investigate multi-view representation learning for images using contrastive losses. 
They demonstrate that when multiple competing features exist that redundantly predict the match between two views, contrastive models tend to focus on learning the easy-to-represent features while suppressing other task-relevant information.
This results in contrastive losses mainly capturing the easy features, even if all task-relevant information is shared between the two views, suppressing the remaining relevant information.

Several optimization objectives have been introduced to either maximize the lower bound on the \ac{MI} between views and their latent representations~\citep{oord2018representation, bachman2019learning, hjelm2019learning, tian2020contrastive} or minimize the \ac{MI} between representations of views while keeping the task-relevant information~\citep{federici2020learning, lee2021compressive}.
To learn more task-relevant information that either might not be shared between views or that is compressed by a contrastive loss, several works proposed additional reconstruction objectives to maximize the \ac{MI} between the latent representation and input data~\citep{tsai2021self, wang2022rethinking, li2023addressing, bleeker2023reducing}.
\cite{liang2023factorized} introduce a multimodal contrastive objective that factorizes the representations into shared and unique information, while also removing task-irrelevant information by minimizing the upper bound on \ac{MI} between similar views.

\header{Vision-language Representation Learning}
The goal of \ac{VL} representation learning is to combine information from the visual and textual modalities into a joint representation or learn coordinated representations~\citep{baltrusaitis2019_multimodal, guo2019_deep}.
The representation learning approaches can be separated into several groups.

\emph{Contrastive methods} represent one prominent category of \ac{VL} representation methods. 
The approaches in this group are typically dual encoders.
Early methods in this category are trained from scratch; for instance, \citep{frome2013devise} proposed a \ac{VL} representation learning model that features a skip-gram language model and a visual object categorization component trained with hinge rank loss. 
Another subgroup of methods uses a \emph{dual-encoder} with a hinge-based triplet loss~\citep{kiros2014unifying, li2019visual, lee2018stacked}.
\cite{kiros2014unifying} use the loss for training a CNN-RNN dual encoder. 
\citet{li2019visual} leverage bottom-up attention and graph convolutional networks~\citep{kipf2017_semi} to learn the relationship between image regions. 
\cite{lee2018stacked} add stacked cross-attention to use both image regions and words as context.

More recently, contrastive approaches involve transformer-based dual-encoders trained with more data than the training data from the evaluation set(s).
ALBEF~\citep{li2021align} propose to contrastively align unimodal representations before fusion, while X-VLM~\citep{zeng2022multi} employs an additional cross-modal encoder to learn fine-grained \ac{VL} representations.
Florence~\citep{yuan2021florence} leverages various adaptation models for learning fine-grained object-level representations.
CLIP~\citep{radford2021learning}, a scaled-up dual-encoder, is pre-trained on the task of predicting which caption goes with which image.
ALIGN~\citep{jia2021scaling} uses a simple dual-encoder trained on over a billion image alt-text pairs.
FILIP~\citep{yao2022flip} is a transformer-based bi-encoder that features late multimodal interaction meant to capture fine-grained representations.
SLIP~\citep{mu2022slip} combines language supervision and image self-supervision to learn visual representations without labels.
DeCLIP~\citep{li2022supervision} proposes to improve the efficiency of CLIP pretraining using intra-modality self-supervision, cross-modal multi-view supervision, and nearest neighbor supervision.

Another line of work includes learning \ac{VL} representations using models that are inspired by BERT~\citep{devlin2019_bert}.
ViLBERT~\citep{lu2019vilbert} and LXMERT~\citep{tan2019lxmert} expand upon BERT by introducing a two-stream architecture, where two transformers are applied to images and text independently, which is fused by a third transformer in a later stage. 
B2T2~\citep{alberti2019fusion}, VisualBERT~\citep{li2019visualbert}, Unicoder-VL~\citep{li2019unicoder}, VL-BERT~\citep{su2019vl}, and UNITER~\citep{chen2020uniter} propose a single-stream architecture, where a single transformer is applied to both images and text.
Oscar~\citep{li2020oscar} uses caption object tags as anchor points that are fed to the transformer alongside region features.
BEIT-3~\citep{wang2022image} adapt multiway transformers trained using cross-entropy loss~\citep{bao2022vlmo}.

Another category of methods for learning \ac{VL} representations are generative methods, that imply learning \ac{VL} representation by generating new instances of one modality conditioned on the other modality. For instance, BLIP~\citep{li2022blip} bootstraps captions by generating synthetic captions and filtering out the noisy ones; BLIP-2~\citep{li2023blip} bootstraps \ac{VL} representation learning and, subsequently, vision-to-language generative learning.
On the other hand, \citet{tschannen2023image} propose to pretrain a encoder-decoder architecture via the image captioning task.

\header{Shortcut Learning} \cite{geirhos2020shortcut} define shortcuts in deep neural networks as ``decision rules that perform well on standard benchmarks but fail to transfer to more challenging testing conditions, such as real-world scenarios.''
In the context of deep learning, a shortcut solution can also be seen as a discrepancy between the features that a model has learned during training and the intended features that a model should learn to perform well during evaluation. 
For example, shortcuts might be features that minimize the training objective but are much easier to detect than the intended features that are relevant to the evaluation task.
Shortcut learning can be caused by biases in the dataset or inductive biases in either the network architecture or training objective.

\cite{hermann2020what} design a dataset with multiple predictive features, where each feature can be used as a label for an image classification task.
The authors show that in the presence of multiple features that each redundantly predicts the target label, the deep neural model chooses to represent only one of the predictive features that are the easiest to detect, i.e., the model favors features that are easy to detect over features that are harder to discriminate.
Next to that, they show that features that are not needed for a classification task, are in general suppressed by the model instead of captured in the learned latent representations.

\cite{robinson2021can} show that contrastive losses can have multiple local minima, where different local minima can be achieved by suppressing features from the input data (i.e., the model learns a shortcut by not learning all task-relevant features). 
To mitigate the shortcut learning problem, \cite{robinson2021can} propose \acl{IFM}, a method that perpetuates the features of positive and negative samples during training to encourage the model to capture different features than the model currently relies on. 

\cite{scimeca2022which} design an experimental set-up with multiple shortcut cues in the training data, where each shortcut is equally valid w.r.t. predicting the correct target label. The goal of the experimental setup is to investigate which cues are preferred to others when learning a classification task.  

\Acf{LTD} is a method to reduce predictive feature suppression (i.e.,  shortcuts) for resource-constrained contrastive \ac{ICR} by reconstructing the input caption in a non-auto-regressive manner.
\cite{bleeker2023reducing} argue that most of the task-relevant information for the \ac{ICR} task is captured by the text modality.
Hence, the focus is on the reconstruction of the text modality instead of the image modality.
\cite{bleeker2023reducing} add a decoder to the learning algorithm, to reconstruct the input caption.
Instead of reconstructing the input tokens, the input caption is reconstructed in a non-autoregressive manner in the latent space of a \acl{SBERT}~\citep{reimersers2019sentence, song2020mpnet} model.
\ac{LTD} can be implemented as an optimization constraint and as a dual-loss.
\cite{li2023addressing} show that contrastive losses are prone to feature suppression. 
They introduce predictive contrastive learning (PCL), which combines contrastive learning with a decoder to reconstruct the input data from the latent representations to prevent shortcut learning.

\cite{adnan2022monitoring}  measure the \ac{MI} between the latent representation and the input as a domain agnostic metric to find where (and when) in training a network relies on shortcuts in the input data. 
Their main finding is that, in the presence of shortcuts, the \ac{MI} between the input data and the latent representation of the data is lower than without shortcuts in the input data. 
Hence, the latent representation captures less information of the input data in the presence of shortcuts and mainly relies on shortcuts to predict the target.

\header{Our Focus}
In this work, we focus on the problem of shortcut learning for \ac{VL} in the context of multi-view \ac{VL} representation learning with multiple captions per image. 
In contrast with previous (uni-modal) work on multi-view learning, we consider different captions matching to the same image as different \textit{views}.
We examine the problem by introducing a framework of synthetic shortcuts designed for \ac{VL} representation learning, which allows us to investigate the problem in a controlled way.
For our experiments, we select two prevalent \ac{VL} models that are solely optimized with the InfoNCE loss: 
CLIP, a large-scale pre-trained model, and VSE++, a model trained from scratch.
We select models that are solely optimized with a contrastive loss, to prevent measuring the effect of other optimization objectives on the shortcut learning problem. 

%% file: sections/08_conclusion.tex

\section{Conclusion}
\label{sec:conclusion}

In this work, we focus on the shortcut learning problem of contrastive learning in the context of \acf{VL} representation learning with multiple captions per image.
We have proposed \acf{SVL}: a training and evaluation framework to examine the problem of shortcut learning in a controlled way. 
The key component of this framework is synthetic shortcuts that we add to image-text data. 
Synthetic shortcuts represent additional, easily identifiable information that is shared between images and captions.
We fine-tune CLIP and train a VSE++ model from scratch using our training framework to evaluate how prone contrastive \ac{VL} models are to shortcut learning.
Next, we have evaluated how shortcut learning can be partially mitigated using \acl{LTD} and \acl{IFM}.

\header{Main Findings} We have conducted experiments on two distinct \ac{VL} models, CLIP and VSE++, and have evaluated the performance on \ac{Flickr30k} and \ac{MS-COCO}.
We have found that when training with unique shortcuts, CLIP suppresses pre-trained features in favor of the shortcuts.
	VSE++ only learns to represent the shortcuts, when using unique shortcuts, showing that none of the remaining task-relevant (both shared and unique) information is captured by the encoders when training a model from scratch.
	When using \textit{$n$ bits of shortcuts}, we have shown that the more bits we use, the more the contrastive \ac{VL} models rely on the synthetic shortcuts.
	Our results demonstrate that contrastive \ac{VL} methods tend to depend on easy-to-learn discriminatory features shared among images and all matching captions while suppressing the remaining task-relevant information. 
	Next, we have evaluated two methods for reducing shortcut learning on our framework of synthetic shortcuts for image-caption datasets.
	Both methods partially mitigate shortcut learning when training and evaluating with our shortcut learning framework.
	These findings show that our framework is a challenging framework to study and evaluate shortcut learning for contrastive \ac{VL} and underline the complexity of our framework in studying and evaluating shortcut learning within the context of contrastive \ac{VL} representation learning.

\header{Implications} The implications of our findings are twofold.
First, we examine the limitations of contrastive optimization objectives for \ac{VL} representation learning, demonstrating that they predominantly capture features that are easily discriminable but may not necessarily constitute task-optimal representations.
Second, our work contributes a novel framework for investigating shortcut learning problem in the context of \ac{VL} representation learning with multiple captions per image, providing insights into the extent to which models rely on shortcuts when they are available and how existing shortcut reduction methods are capable of reducing shortcut learning when training with our framework.

\header{Limitations} Some of the limitations of our work are related to the fact that we focused on two specific models, one optimization objective (InfoNCE), and two datasets, and the generalizability of our findings to other \ac{VL} models, optimization objectives, and datasets warrants further exploration. 
Additionally, the synthetic shortcuts introduced in this work are not dependent on image-caption pairs. 
Our training and evaluation setup shows that, in the presence of shortcuts in the training data, contrastive \ac{VL} models mainly rely on the easy-to-detect shortcut features, which indicates that the InfoNCE loss cannot learn tasks-optimal representations for \ac{VL} tasks when multiple captions are used for training.
However, it remains unclear to what degree the unique information of the captions is captured by the contrastive loss \ac{VL} models.

\header{Future Work} 
We suggest working on the development of optimization objectives that specifically address the shortcut learning problem for \ac{VL} training with multiple captions per image.
We also suggest extending our synthetic shortcuts for image-caption datasets to a framework with unique shortcut information per caption. 
By having unique shortcut information per caption, it becomes possible to measure how much of the shared/caption-specific shortcut information is captured by encoder models.
Another future direction includes investigating alternative training strategies or loss functions to further mitigate shortcut learning problems. 
Another promising direction for future work includes the improvement of existing methods or the exploration of novel techniques that address the limitations of existing shortcut reduction methods, potentially through the combination of multiple approaches.
Extending the SVL framework to better capture nuances and complexities of natural data is another important  direction that would facilitate a more comprehensive understanding of the implications of shortcut learning in real-world scenarios and datasets.

%% file: sections/09_impact.tex

\section{Broader Impact}
This paper motivates and introduces a framework for investigating the problem of shortcut learning for contrastive \ac{VL} representation learning with multiple captions per image in a controlled way.
It also examines how two shortcut learning reduction methods perform on the proposed framework.
Overall, the framework provides a tool for analyzing and understanding the problem of shortcut learning in the context of contrastive \ac{VL} representation learning; it can be used in various settings that require deeper insight into the quality of learned \ac{VL} representations.

We should be aware that the reliance on shortcuts in \acp{VLM} poses ethical concerns with potential real-world implications. 
Models that learn shortcuts may overlook nuanced details in images and text, leading to biased or inaccurate outcomes. 
Furthermore, the transparency and explainability of \acp{VLM} are crucial considerations. 
Models that rely on shortcuts may make decisions based on features that are not easily interpretable or explainable to users. 
This lack of transparency can diminish trust in AI systems.

%% file: sections/appendix/appendix.tex

\input{sections/appendix/notation}
\input{sections/appendix/problem}
\input{sections/appendix/background}
\input{sections/appendix/experimental_setup}
\input{sections/appendix/losses}

%% file: sections/appendix/notation.tex

\clearpage
\section{Notation}
\label{appendix:notation}

\begin{table}[htb!]
\centering
	\setlength{\tabcolsep}{1pt}
	\setlength\extrarowheight{-1.25pt}
\caption{Notation used in the paper.}
\label{tab:notation}
\begin{tabular}{p{.14\textwidth} p{.86\textwidth}} 
\toprule
\textbf{Symbol} & \textbf{Description} \\ \midrule
$\infonce{}$ & InfoNCE loss \\ \cmidrule{2-2}
$\ltd{}$ & Loss that combines InfoNCE and \acf{LTD} \\ \cmidrule{2-2}
$\ifm{}$ & Loss that combines InfoNCE and \acf{IFM} \\ \cmidrule{2-2}
$\mathcal{L}_{recon}$  & Reconstruction loss \\ \cmidrule{2-2}
$\mathcal{D}$ & Dataset $\mathcal{D}$ that comprises $N$ image-caption tuples: $\mathcal{D} = \left\{\left(\img{i}, \{\capt{i}{j}\}_{j=1}^{k} \right) \right\}_{i=1}^{N}$; $i$-th image-caption tuple in the dataset $\mathcal{D}$ consist out of an image $\img{i}$ and $k$ associated captions $\{\capt{i}{j}\}_{j=1}^{k}$ 
\\ \cmidrule{2-2}
$\mathcal{B}$ & Batch of image-caption pairs \\ \cmidrule{2-2}
$\img{}$ & Image \\ \cmidrule{2-2}
$\capt{}{}$ & Caption \\ \cmidrule{2-2}
$\zimg{}{}$ & Latent representation of image $\img{}$ \\ \cmidrule{2-2}
$\zcapt{}{}$ & Latent representation of caption $\capt{}{}$ \\ \cmidrule{2-2}
$\zcaptsuff{}{}$ & Latent representation of caption $\capt{}{}$ that is sufficient for image $\img{}$ \\ \cmidrule{2-2}
$\zimgsuff{}$ & Latent representation of image $\img{}$ sufficient for caption $\capt{}{}$ 
\\ \cmidrule{2-2}
$\zminsuff{}$ & Latent representation of image $\img{}$ that is minimal sufficient for caption $\capt{}{}$ 
\\ \cmidrule{2-2}
$\zoptimal{}{K}$ & Latent representation of image $\img{}$ that is optimal for set of captions $K$ given  task $T$  \\ \cmidrule{2-2}
$R$ & Task-relevant information \\ \cmidrule{2-2}
$\neg R$ & Task-irrelevant information \\ \cmidrule{2-2}
$C$ & Task-relevant information specific for a caption $\capt{}{}$ \\ \cmidrule{2-2}
$\sct{}$ & Synthetic shortcut \\ \cmidrule{2-2}
$S$ & Original shared information \\ \cmidrule{2-2}
$S^{+}$ & Shared information that includes synthetic shortcut  \\ \cmidrule{2-2}
$R^{+}$ & Task-relevant information that contains synthetic shortcut \\ \cmidrule{2-2}
$f_{\theta}(\cdot)$ & Image encoder parametrised by $\theta$; takes image $\img{}$ as input and returns its latent representation $\zimg{}{}$: $\zimg{}{} := f_{\theta}(\img{})$  
\\ \cmidrule{2-2}
$g_{\phi}(\cdot)$ & Caption encoder parametrised by $\phi$; takes caption $\capt{}{}$ as input and returns its latent representation $\zcapt{}{}$: $\zcapt{}{} := g_{\phi}(\zcapt{}{})$ 
\\  \cmidrule{2-2}
$\tau$ & Temperature paramater of $\infonce{}$ \\  \cmidrule{2-2}
$\epsilon$ & Perturbation budget for $\mathcal{L}_{\text{IFM}}$\\  \cmidrule{2-2}
$\eta$ & Reconstruction bound  for $\mathcal{L}_{\text{LTD}}$ \\
\bottomrule
\end{tabular}
\end{table}

%% file: sections/appendix/problem.tex

\section{Problem Definition and Assumptions}
\label{appendix:problem}

In this work, we solely focus on contrastive \ac{VL} representation learning. 
We work in a setting where we investigate the problem by fine-tuning a large pre-trained foundation model \citep[CLIP,][]{radford2021learning} and training a resource-constrained image-text method from scratch \citep[VSE++,][]{faghri2018improving}.
We train and evaluate using two benchmark datasets where multiple captions per image are available: \acl{Flickr30k} \citep{young2014image} and \acl{MS-COCO} \citep{lin2014microsoft}. 
Both datasets come with 5 captions per image.
We work in a dual-encoder setup, i.e., we have a separate image and caption encoder, which do not share parameters.

\subsection{Evaluation Task}
\label{appendix:task}

The \acf{ICR} evaluation task, consists of two sub-tasks: \acf{i2t} and \acf{t2i} retrieval.
In \ac{ICR}, either an image or a caption is used as  a query and the goal is to rank a set of candidates in the other modality.
In this work, we follow the standard \ac{ICR} evaluation procedure \citep[see, e.g.,][]{faghri2018improving, lee2018stacked, li2019visual}.
The evaluation metric for the \ac{ICR} task is Recall$@k$, with $k=\{1, 5, 10\}$.
For \ac{t2i} retrieval, there is one matching/positive image per query caption (when using the \ac{Flickr30k} or \ac{MS-COCO} or dataset). 
Hence, the Recall$@k$ metric represents how often the correct image is present in the top-$k$ of the ranking. 
For \ac{i2t} retrieval, however, there are 5 matching captions per image. 
Therefore, only the highest-ranked correct caption is taken into account when measuring the Recall$@k$ (i.e., in the highest-ranked caption present in the top $k$).
Standard practice to select the best model checkpoint during training is to use the \emph{recall sum} (rsum) as a validation metric. 
The recall sum is the sum of recall at 1, 5, and 10, for both \ac{i2t} and \ac{t2i}.
Therefore, the maximum value of the recall sum is 600. 

\subsection{Assumptions}
\label{app:assumptions}

Throughout this work, we rely on several assumptions about the problem definition. 
Our assumptions are defined at the level of an image-text tuple. Following Section~\ref{sec:background}, we formalize the assumptions on the case where one image is associated with two captions: $\left(\img{}, \{ \capt{}{A}, \capt{}{B} \} \right)$.

\begin{assumption}
\label{asm:shared-unique}
	Each caption in the tuple contain information that is distinct from the other captions in the tuple and all captions and image in the tuple contain shared and unique information:
	\begin{align*}
		&I(\img{}; \capt{}{A}; \capt{}{B}) > 0 \\
		&I(\img{}; \capt{}{A} \mid \capt{}{B}) > 0,\  I(\img{}; \capt{}{B} \mid \capt{}{A}) > 0\ \text{and}\  I(\capt{}{A}; \capt{}{B} \mid \img{}) > 0 \\
		&H(\img{} \mid \capt{}{A}, \capt{}{B}) > 0,\ H(\capt{}{A} \mid \img{}, \capt{}{B}) > 0\ \text{and}\ H(\capt{}{B} \mid \img{}, \capt{}{A}) > 0.
	\end{align*}
\end{assumption}

\begin{assumption}
\label{asm:task-relevant-info}
	Task-relevant information $R$ is the combination of all the information shared between an image and each caption in the tuple:
	\begin{align*}
		R = I(\img{}; \capt{}{A} \mid \capt{}{B})  + I(\img{}; \capt{}{B} \mid \capt{}{A}) + I(\img{}; \capt{}{A}; \capt{}{B}).
	\end{align*}
\end{assumption}

%% file: sections/appendix/background.tex

\section{Analysis of Contrastive Learning for Multiple Captions per Image}
\label{app:analysis-of-contrastive}

\begin{thm}{1}[Suboptimality of contrastive learning with multiple captions per image]
	\label{thm:suboptimality-app}	
	Given an image $\img{}{}$, a set of matching captions $\mathcal{C} = \{ \capt{}{A}, \capt{}{B} \}$, and a contrastive learning loss function $\infonce{}$ that optimizes for task $T$, 
	image representations learned during contrastive learning will be minimal sufficient and will never be task-optimal image representations.
	More formally, assume that:
	\begin{itemize}
		\item[$(H_1)$] $\forall i, j \in \{A, B\}\ \text{such that}\ i \neq j,\ I(\zminsuff{i} ; \capt{}{i}) = I(\img{}{} ; \capt{}{i} \mid \capt{}{j}) + I(\img{}{} ; \capt{}{i} ; \capt{}{j})$.
		\item[$(H_2)$] $\exists i, j \in \{A, B\}\ \text{with}\ i \neq j\ \text{such that}\ I(\img{}{} ; \capt{}{i} \mid \capt{}{j}) > 0$.
	\end{itemize}
	Then the following holds:
	\begin{itemize}
		\item[$(T_2)$] $\exists \ i \in \{A, B\}\ \text{such that}\ I(\zoptimal{}{\mathcal{C}} ; \capt{}{A}\capt{}{B}) > I(\zminsuff{i} ; \capt{}{i})$.
	\end{itemize} 
\end{thm}

\begin{proof} 
Following Eq.~\ref{eq:task_relevant} we define a task-optimal representation of an image $\img{}{}$ w.r.t.\ all matching captions in $\mathcal{C}$ as:
	\begin{align*}
		I(\zoptimal{}{\mathcal{C}} ; \capt{}{A}\capt{}{B} )
		 =
		 \underbrace{ I( \img{}{} ; \capt{}{A} \mid \capt{}{B}) }_{\textit{$C_A$}}
		+
		\underbrace{I( \img{}{} ; \capt{}{B} \mid \capt{}{A})}_{\textit{$C_B$}}
		+
		\underbrace{I( \img{}{} ; \capt{}{A} ;  \capt{}{B})}_{\textit{$S$}}.
	\end{align*}

Furthermore, following Definition~\ref{def:zimgminsuff}, we define minimal sufficient representations of image $\img{}{}$ w.r.t.\ each matching caption in $\mathcal{C}$ as a combination of caption-specific and shared information:
	\begin{align*}
		I(\zminsuff{A} ; \capt{}{A})
		 & =
		 \underbrace{I(\img{}{} ; \capt{}{A} \mid \capt{}{B})}_{\textit{$C_A$}}
		+
		\underbrace{I(\img{}{} ; \capt{}{A} ; \capt{}{B})}_{\textit{$S$}} 
		\\
		I(\zminsuff{B} ; \capt{}{B})
		& =
		 \underbrace{I(\img{}{} ; \capt{}{B} \mid \capt{}{A})}_{\textit{$C_B$}}
		+
		\underbrace{ I(\img{}{} ; \capt{}{A} ; \capt{}{B}) }_{\textit{$S$}}.
	\end{align*}
Following assumption $H_2$, for at least one caption $\capt{}{} \in \mathcal{C}$ associated with the image $\img{}{}$, caption-specific information is positive. Therefore, we consider two cases:
\begin{itemize}
\item If caption-specific information of $\capt{}{A}$ is positive, that is, if $I( \img{}{} ; \capt{}{A} \mid \capt{}{B}) > 0$:
	\begin{align*}
	\underbrace{I( \img{}{} ; \capt{}{A} \mid \capt{}{B}) + I( \img{}{} ; \capt{}{B} \mid \capt{}{A}) + I( \img{}{} ; \capt{}{A} ;  \capt{}{B})}_{\textit{$(\zoptimal{}{\mathcal{C}} ; \capt{}{A}\capt{}{B} )$}}
	>
	\underbrace{ I(\img{}{} ; \capt{}{B} \mid \capt{}{A}) + I(\img{}{} ; \capt{}{A} ; \capt{}{B}) }_{\textit{$I(\zminsuff{B} ; \capt{}{B})$} } \Rightarrow \\
	\Rightarrow
	I( \zoptimal{}{\mathcal{C}} ; \capt{}{A}\capt{}{B} ) > I(\zminsuff{B} ; \capt{}{B}).
	\end{align*}

\item Similarly, if caption-specific information of $\capt{}{B}$ is positive, that is, if $I( \img{}{} ; \capt{}{B} \mid \capt{}{A}) > 0$:

	\begin{align*}
	\underbrace{I( \img{}{} ; \capt{}{A} \mid \capt{}{B}) + I( \img{}{} ; \capt{}{B} \mid \capt{}{A}) + I( \img{}{} ; \capt{}{A} ;  \capt{}{B})}_{\textit{$(\zoptimal{}{\mathcal{C}} ; \capt{}{A}\capt{}{B} )$}}
	>
	\underbrace{ I(\img{}{} ; \capt{}{A} \mid \capt{}{B}) + I(\img{}{} ; \capt{}{A} ; \capt{}{B}) }_{\textit{$I(\zminsuff{A} ; \capt{}{A})$} } \Rightarrow \\
	\Rightarrow
	I( \zoptimal{}{\mathcal{C}} ; \capt{}{A}\capt{}{B} ) > I(\zminsuff{A} ; \capt{}{A}).
	\end{align*}
\end{itemize}

Therefore, we show that in a setup where a single image is associated with multiple captions, and given at least one caption contains caption-specific information, image representations learned contrastively w.r.t.\ associated captions would contain less information than task-optimal image representation: $\exists \ i \in \{A, B\}\ \text{such that}\ I(\zoptimal{}{\mathcal{C}} ; \capt{}{A}\capt{}{B}) > I(\zminsuff{i} ; \capt{}{i})$.
\end{proof}

%% file: sections/appendix/experimental_setup.tex

\section{Experimental Setup}
\label{app:experimental}

\subsection{Datasets}
\label{app:experimental-datasets}

\headernodot{Flickr30k} consists of \numprint{31000} images annotated with 5 matching captions~\citep{young2014image}. 

\headernodot{MS-COCO} consists of \numprint{123287} images, each image annotated with 5 matching captions~\citep{lin2014microsoft}.  
The original dataset was introduced for large-scale object recognition.

For both datasets, we use the training, validation, and test splits from~\citep{karpathy2015deep}.

\subsection{Models}
\label{app:experimental-models}

We use CLIP and VSE++. Both consist of an image and a text encoder that do not share parameters.

\headernodot{CLIP} is a large-scale image-text foundation model~\citep{radford2021learning}.
The model is pre-trained on a collection of \numprint{400} million image-text pairs collected from the Web. 
The encoders are pre-trained using a contrastive loss (InfoNCE) on image-text pairs.
The text encoder of consists of a 12-layer transformer model, described in~\citep{radford2019language}.
As for the image encoder, CLIP utilizes various model backbones, such as ResNet~\citep{he2016deep}  and Vision Transformer~\citep{dosovitskiy2021image}.
In this work, we use the ResNet-50  (‘RN50’) variant of the CLIP image encoder.\footnote{\url{https://github.com/openai/CLIP/}}
The CLIP encoders are trained to jointly understand images and text.
Therefore, the learned representations generalize to a wide range of different zero-shot (visual) evaluation tasks, such as classification, without task-specific fine-tuning, by using textual prompts. 

\headernodot{VSE++} is an image-caption encoder trained from scratch~\citep{faghri2018improving}.
The model features a triplet loss function with a margin parameter $\alpha=0.2$.
The text encoder is a one-layer gated recurrent unit (GRU)~\citep{cho2014_learning}.
The available image encoder configurations are ResNet-152~\citep{he2016deep} and VGG19~\citep{simonyan2015_very_deep}.
In this work, we use ResNet-152.

\subsection{Training}
\label{app:experimental-training}

\header{CLIP} To fine-tune CLIP, we follow~\citep{yuksekgonul2023when}.
All models are fine-tuned for 5 epochs.
We employ a cosine-annealing learning rate schedule, with a base learning rate of $2e-5$, and 100 steps of warm-up.
As an optimizer, we use AdamW~\citep{loshchilov2019decoupled} with a gradient clipping value of $2$.
For the InfoNCE loss, we use the logit-scale (i.e., temperature $\tau$) from the pre-trained CLIP model and fine-tune the logit-scale end-to-end along with the rest of the model parameters.

\header{VSE++} The model is trained for 30 epochs using a linear learning rate schedule with a base learning rate of $2e-4$. We use the Adam optimizer~\citep{kingma2014_adam} with a gradient clipping value of $2$. Instead of the triplet loss, we use the InfoNCE loss similar to~\cite{radford2021learning},

For both models, instead of selecting the best-performing model based on the validation set scores, we use the final checkpoint at the end of training.

\subsection{Shortcut Sampling}
\label{app:experimental-shortcutsampling}

Our goal is to add the shortcuts in a manner that preserves the original information of the images and captions. 
For the captions, we append the shortcut at the end of the captions.
In order to prevent a tokenizer from tokenizing the shortcut into a single token, we insert spaces between each number of the shortcut.
For the images, we place the numbers of the shortcuts at the top of the images, evenly spaced across the entire width of the images (to make sure the shortcut is evenly spaced across the feature map of the image).
We always use 6 digits to represent a shortcut.  
If a shortcut number contains fewer than 6 digits, we fill the remaining positions with zeros for padding.
For the MNIST images, we always sample a random image from the set of images representing the number that belongs to (also during evaluation), to prevent overfitting on specific MNIST images. 
In Figure~\ref{fig:shortcut_examples}, we provide four examples of image-caption pairs with randomly added shortcuts. 
The examples in Figure~\ref{fig:shortcut_examples} show 
\begin{enumerate*}[label=(\roman*)]
	\item how synthetic shortcuts are added to the image and the caption, and 
	\item that the shortcuts preserve the original (task-relevant) information of the images and captions.
\end{enumerate*}

\begin{figure}[t!]
	\centering
	\begin{subfigure}[b]{0.475\textwidth}
		\centering
		\includegraphics[width=\textwidth, height=\textwidth]{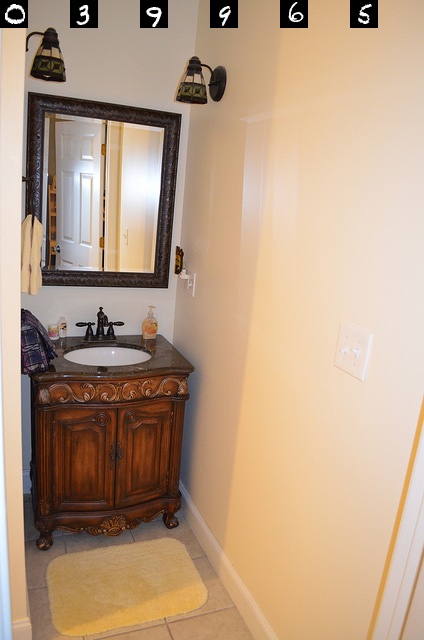}
		\caption{\textbf{Caption}: ``A bathroom sink with wood finish cabinets.  0 3 9 9 6 5.''}%
		\label{fig:shortcut1}
	\end{subfigure}
	\hfill
	\begin{subfigure}[b]{0.475\textwidth}  
		\centering 
		\includegraphics[width=\textwidth, height=\textwidth]{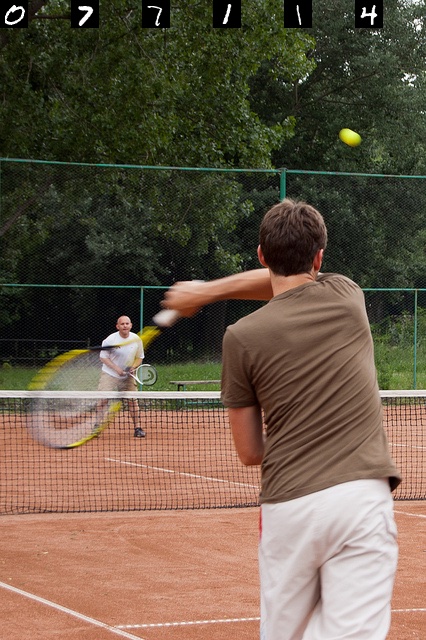}
		\caption{\textbf{Caption}: ``A guy in a brown shirt has just hit a tennis ball.  0 7 7 1 1 4.''}%
		\label{fig:shortcut2}
	\end{subfigure}
	\vskip\baselineskip
	\begin{subfigure}[b]{0.475\textwidth}   
		\centering 
		\includegraphics[width=\textwidth, height=\textwidth]{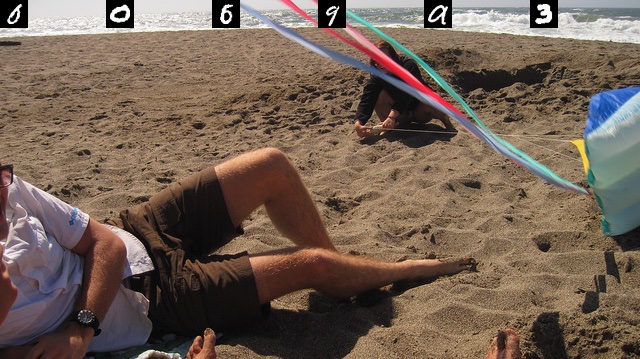}
		\caption{\textbf{Caption}: ``A man in shorts is lying on the beach. 0 0 6 9 9 3.''}%
		\label{fig:shortcut3}
	\end{subfigure}
	\hfill
	\begin{subfigure}[b]{0.475\textwidth}   
		\centering 
		\includegraphics[width=\textwidth, height=\textwidth]{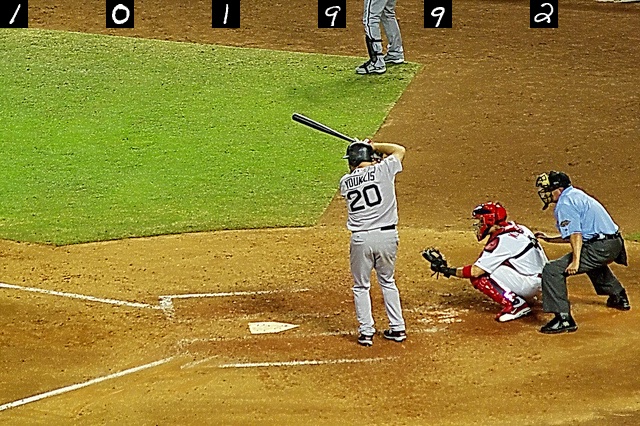}
		\caption{\textbf{Caption}: ``A player up to bat in a baseball game. 1 0 1 9 9 2.''}%
		\label{fig:shortcut4}
	\end{subfigure}
	\caption{Four random samples from the \acs{MS-COCO} dataset including shortcuts added on both the image and caption.
	}
	\label{fig:shortcut_examples}
\end{figure}

%% file: sections/appendix/losses.tex

\section{Optimization Objectives}\label{app:losses}

\subsection{InfoNCE}\label{app:info}
In this work, we use InfoNCE loss, $\infonce{}$~\citep{oord2018representation}.
Given a dual-encoder setup, we optimize a model in two directions: \acf{i2t} and \acf{t2i}.
The loss is defined as follows: 

\begin{subequations}
\begin{align*}
	\infonce{i2t} = \frac{1}{|\mathcal{B}|} \sum_{i \in \mathcal{B}}^{}\log \frac{
		\exp(\zimg{i}{}\zcapt{i}{}/\tau)
	}
	{\exp(\zimg{i}{}\zcapt{i}{}/\tau) + \sum_{j \neq i }^{} \exp(\zimg{i}{}\zcapt{j}{}/\tau)}, 
\end{align*}
\begin{align*}
	\infonce{i2t} = \frac{1}{|\mathcal{B}|} \sum_{i \in \mathcal{B}}^{}\log \frac{
		\exp(\zimg{i}{}\zcapt{i}{}/\tau)
	}
	{\exp(\zimg{i}{}\zcapt{i}{}/\tau) + \sum_{j \neq i }^{} \exp(\zimg{j}{}\zcapt{i}{}/\tau)}, 
\end{align*}
\begin{align*}
 	\infonce{} =  \frac{1}{2}  	\infonce{i2t} +  \frac{1}{2}  	\infonce{t2i}.
\end{align*}
\end{subequations}

\subsection{Latent Target Decoding}\label{app:ltd}

\Acf{LTD}~\citep{bleeker2023reducing} is an optimization objective that reduces predictive feature suppression for resource-constrained \ac{VL} methods.
\ac{LTD} consists of $\infonce{}$ and a reconstruction loss $ \mathcal{L}_{recon}$, which reconstructs the input caption from the latent representation $\zcapt{}{}$.

In \citep{bleeker2023reducing}, \ac{LTD} is implemented in two ways. 
Firstly, as a dual optimization objective:
\begin{align*}
	\ltd = \infonce{} + \beta \mathcal{L}_{recon}.
\end{align*}
Secondly, as an optimization constraint in combination with gradient descent by using the method of Lagrange multipliers:
\begin{align*}
	\max_{\lambda} \min_{}  \ltd=  \infonce{}  + \lambda \left( \frac{\mathcal{L}_{recon}}{\eta} - 1 \right).
\end{align*}

This optimization objective is minimized w.r.t. model parameter, while also being maximized w.r.t. $\lambda$. 
The value of $\lambda$ is automatically tuned by gradient ascent, such that the reconstruction bound $\eta$ is met.
In this work, we use both \ac{LTD} as a dual optimization objective and an optimization constraint. 
We select the loss with the highest evaluation scores on the validation set for evaluation.

\subsection{Implicit Feature Modification}\label{app:ifm}

\Acf{IFM}~\citep{robinson2021can} is a contrastive loss, with an additional perturbation budget $\epsilon$.
\ac{IFM} perturbs the logits value of the similarity scores between the images and captions, such that the model avoids using shortcut solutions for a correct similarity score.
\ac{IFM} subtracts $\epsilon/\tau$ from the positive logit values and adds $\epsilon/\tau$ to the negative logits values.

\begin{subequations}
\begin{align*}
	\mathcal{L}_{\text{IFM}}^{t2i} = \frac{1}{|\mathcal{B}|} \sum_{i \in \mathcal{B}}^{}\log \frac{
		\exp(
		(\zimg{i}{}\zcapt{i}{} - \epsilon)
		/\tau
		)
	}
	{\exp((\zimg{i}{}\zcapt{i}{}) - \epsilon)/\tau) + \sum_{j \neq i }^{} \exp((\zimg{i}{}\zcapt{j}{} + \epsilon)/\tau)},
\end{align*}
\begin{align*}
	\mathcal{L}_{\text{IFM}}^{i2t} = \frac{1}{|\mathcal{B}|} \sum_{i \in \mathcal{B}}^{}\log \frac{
		\exp(
		(\zimg{i}{}\zcapt{i}{} - \epsilon)
		/\tau
		)
	}
	{\exp((\zimg{i}{}\zcapt{i}{}) - \epsilon)/\tau) + \sum_{j \neq i }^{} \exp((\zimg{j}{}\zcapt{i}{} + \epsilon)/\tau)},
\end{align*}
\begin{align*}
	\mathcal{L}_{\text{IFM}}^{} =  \frac{1}{2}  	\mathcal{L}_{\text{IFM}}^{t2i}+  \frac{1}{2}  		\mathcal{L}_{\text{IFM}}^{i2t},
\end{align*}
\begin{align*}
\ifm =  \frac{1}{2}  	\mathcal{L}_{\text{IFM}}^{}  +  \frac{1}{2}  \infonce{}.
\end{align*}
\end{subequations}

Similar to ~\cite{robinson2021can}, we combine \ac{IFM} and the InfoNCE in a dual optimization objective.

%% file: main.bbl
\begin{thebibliography}{72}
\providecommand{\natexlab}[1]{#1}
\providecommand{\url}[1]{\texttt{#1}}
\expandafter\ifx\csname urlstyle\endcsname\relax
  \providecommand{\doi}[1]{doi: #1}\else
  \providecommand{\doi}{doi: \begingroup \urlstyle{rm}\Url}\fi

\bibitem[Adnan et~al.(2022)Adnan, Ioannou, Tsai, Galloway, Tizhoosh, and
  Taylor]{adnan2022monitoring}
Mohammed Adnan, Yani Ioannou, Chuan-Yung Tsai, Angus Galloway, H.R. Tizhoosh,
  and Graham~W. Taylor.
\newblock Monitoring shortcut learning using mutual information.
\newblock \emph{arXiv preprint arXiv:2206.13034}, 2022.

\bibitem[Alberti et~al.(2019)Alberti, Ling, Collins, and
  Reitter]{alberti2019fusion}
Chris Alberti, Jeffrey Ling, Michael Collins, and David Reitter.
\newblock Fusion of detected objects in text for visual question answering.
\newblock In \emph{EMNLP}, pp.\  2131--2140, 2019.

\bibitem[Bachman et~al.(2019)Bachman, Hjelm, and
  Buchwalter]{bachman2019learning}
Philip Bachman, R.~Devon Hjelm, and William Buchwalter.
\newblock Learning representations by maximizing mutual information across
  views.
\newblock In \emph{NeurIPS}, pp.\  15509--15519, 2019.

\bibitem[Baltrusaitis et~al.(2019)Baltrusaitis, Ahuja, and
  Morency]{baltrusaitis2019_multimodal}
Tadas Baltrusaitis, Chaitanya Ahuja, and Louis{-}Philippe Morency.
\newblock Multimodal machine learning: {A} survey and taxonomy.
\newblock \emph{{IEEE} Trans. Pattern Anal. Mach. Intell.}, 41:\penalty0
  423--443, 2019.

\bibitem[Bao et~al.(2022)Bao, Wang, Dong, Liu, Mohammed, Aggarwal, Som, Piao,
  and Wei]{bao2022vlmo}
Hangbo Bao, Wenhui Wang, Li~Dong, Qiang Liu, Owais~Khan Mohammed, Kriti
  Aggarwal, Subhojit Som, Songhao Piao, and Furu Wei.
\newblock {VLM}o: Unified vision-language pre-training with
  mixture-of-modality-experts.
\newblock \emph{NeurIPS}, pp.\  32897--32912, 2022.

\bibitem[Bardes et~al.(2022)Bardes, Ponce, and LeCun]{bardes2022vicreg}
Adrien Bardes, Jean Ponce, and Yann LeCun.
\newblock {VICReg}: Variance-invariance-covariance regularization for
  self-supervised learning.
\newblock In \emph{ICLR}, 2022.

\bibitem[Biten et~al.(2022)Biten, Mafla, G{\'{o}}mez, and
  Karatzas]{biten2022_is}
Ali~Furkan Biten, Andr{\'{e}}s Mafla, Llu{\'{\i}}s G{\'{o}}mez, and Dimosthenis
  Karatzas.
\newblock Is an image worth five sentences? {A} new look into semantics for
  image-text matching.
\newblock In \emph{WACV}, pp.\  2483--2492. {IEEE}, 2022.

\bibitem[Bleeker et~al.(2023)Bleeker, Yates, and de~Rijke]{bleeker2023reducing}
Maurits Bleeker, Andrew Yates, and Maarten de~Rijke.
\newblock Reducing predictive feature suppression in resource-constrained
  contrastive image-caption retrieval.
\newblock \emph{Transactions on Machine Learning Research}, 2023.

\bibitem[Chen et~al.(2020{\natexlab{a}})Chen, Kornblith, Norouzi, and
  Hinton]{chen2020simple}
Ting Chen, Simon Kornblith, Mohammad Norouzi, and Geoffrey~E. Hinton.
\newblock A simple framework for contrastive learning of visual
  representations.
\newblock In \emph{ICML}, pp.\  1597--1607, 2020{\natexlab{a}}.

\bibitem[Chen et~al.(2021)Chen, Luo, and Li]{chen2021intriguing}
Ting Chen, Calvin Luo, and Lala Li.
\newblock Intriguing properties of contrastive losses.
\newblock In \emph{NeurIPS}, pp.\  11834--11845, 2021.

\bibitem[Chen et~al.(2015)Chen, Fang, Lin, Vedantam, Gupta, Doll{\'a}r, and
  Zitnick]{chen2015microsoft}
Xinlei Chen, Hao Fang, Tsung-Yi Lin, Ramakrishna Vedantam, Saurabh Gupta, Piotr
  Doll{\'a}r, and C.~Lawrence Zitnick.
\newblock Microsoft {COCO} captions: Data collection and evaluation server.
\newblock \emph{arXiv preprint arXiv:1504.00325}, 2015.

\bibitem[Chen et~al.(2020{\natexlab{b}})Chen, Li, Yu, El~Kholy, Ahmed, Gan,
  Cheng, and Liu]{chen2020uniter}
Yen-Chun Chen, Linjie Li, Licheng Yu, Ahmed El~Kholy, Faisal Ahmed, Zhe Gan,
  Yu~Cheng, and Jingjing Liu.
\newblock {UNITER}: Universal image-text representation learning.
\newblock In \emph{ECCV}, pp.\  104--120, 2020{\natexlab{b}}.

\bibitem[Cho et~al.(2014)Cho, van Merrienboer, G{\"{u}}l{\c{c}}ehre, Bahdanau,
  Bougares, Schwenk, and Bengio]{cho2014_learning}
Kyunghyun Cho, Bart van Merrienboer, {\c{C}}aglar G{\"{u}}l{\c{c}}ehre, Dzmitry
  Bahdanau, Fethi Bougares, Holger Schwenk, and Yoshua Bengio.
\newblock Learning phrase representations using {RNN} encoder-decoder for
  statistical machine translation.
\newblock In \emph{ACL}, pp.\  1724--1734, 2014.

\bibitem[Devlin et~al.(2019)Devlin, Chang, Lee, and Toutanova]{devlin2019_bert}
Jacob Devlin, Ming{-}Wei Chang, Kenton Lee, and Kristina Toutanova.
\newblock {BERT:} {P}re-training of deep bidirectional transformers for
  language understanding.
\newblock In \emph{{NAACL-HLT}}, 2019.

\bibitem[Dosovitskiy et~al.(2021)Dosovitskiy, Beyer, Kolesnikov, Weissenborn,
  Zhai, Unterthiner, Dehghani, Minderer, Heigold, Gelly, Uszkoreit, and
  Houlsby]{dosovitskiy2021image}
Alexey Dosovitskiy, Lucas Beyer, Alexander Kolesnikov, Dirk Weissenborn,
  Xiaohua Zhai, Thomas Unterthiner, Mostafa Dehghani, Matthias Minderer, Georg
  Heigold, Sylvain Gelly, Jakob Uszkoreit, and Neil Houlsby.
\newblock An image is worth 16x16 words: Transformers for image recognition at
  scale.
\newblock In \emph{ICLR}, 2021.

\bibitem[Faghri et~al.(2018)Faghri, Fleet, Kiros, and
  Fidler]{faghri2018improving}
Fartash Faghri, David~J. Fleet, Jamie~Ryan Kiros, and Sanja Fidler.
\newblock {VSE++:} {I}mproving visual-semantic embeddings with hard negatives.
\newblock In \emph{BVCM}, pp.\ ~12, 2018.

\bibitem[Federici et~al.(2020)Federici, Dutta, Forr{\'{e}}, Kushman, and
  Akata]{federici2020learning}
Marco Federici, Anjan Dutta, Patrick Forr{\'{e}}, Nate Kushman, and Zeynep
  Akata.
\newblock Learning robust representations via multi-view information
  bottleneck.
\newblock In \emph{ICLR}, 2020.

\bibitem[Frome et~al.(2013)Frome, Corrado, Shlens, Bengio, Dean, Ranzato, and
  Mikolov]{frome2013devise}
Andrea Frome, Greg~S Corrado, Jon Shlens, Samy Bengio, Jeff Dean, Marc'Aurelio
  Ranzato, and Tomas Mikolov.
\newblock De{V}i{SE}: A deep visual-semantic embedding model.
\newblock \emph{NeurIPS}, pp.\  2121--2129, 2013.

\bibitem[Geirhos et~al.(2020)Geirhos, Jacobsen, Michaelis, Zemel, Brendel,
  Bethge, and Wichmann]{geirhos2020shortcut}
Robert Geirhos, J{\"o}rn-Henrik Jacobsen, Claudio Michaelis, Richard Zemel,
  Wieland Brendel, Matthias Bethge, and Felix~A Wichmann.
\newblock Shortcut learning in deep neural networks.
\newblock \emph{Nature Machine Intelligence}, pp.\  665--673, 2020.

\bibitem[Guo et~al.(2019)Guo, Wang, and Wang]{guo2019_deep}
Wenzhong Guo, Jianwen Wang, and Shiping Wang.
\newblock Deep multimodal representation learning: {A} survey.
\newblock \emph{{IEEE} Access}, 7:\penalty0 63373--63394, 2019.

\bibitem[He et~al.(2016)He, Zhang, Ren, and Sun]{he2016deep}
Kaiming He, Xiangyu Zhang, Shaoqing Ren, and Jian Sun.
\newblock Deep residual learning for image recognition.
\newblock In \emph{CVPR}, pp.\  770--778, 2016.

\bibitem[Hermann \& Lampinen(2020)Hermann and Lampinen]{hermann2020what}
Katherine~L. Hermann and Andrew~K. Lampinen.
\newblock What shapes feature representations? {Exploring} datasets,
  architectures, and training.
\newblock In \emph{NeurIPS}, pp.\  9995--10006, 2020.

\bibitem[Hjelm et~al.(2019)Hjelm, Fedorov, Lavoie{-}Marchildon, Grewal,
  Bachman, Trischler, and Bengio]{hjelm2019learning}
R.~Devon Hjelm, Alex Fedorov, Samuel Lavoie{-}Marchildon, Karan Grewal, Philip
  Bachman, Adam Trischler, and Yoshua Bengio.
\newblock Learning deep representations by mutual information estimation and
  maximization.
\newblock In \emph{ICLR}, 2019.

\bibitem[Jia et~al.(2021)Jia, Yang, Xia, Chen, Parekh, Pham, Le, Sung, Li, and
  Duerig]{jia2021scaling}
Chao Jia, Yinfei Yang, Ye~Xia, Yi{-}Ting Chen, Zarana Parekh, Hieu Pham,
  Quoc~V. Le, Yun{-}Hsuan Sung, Zhen Li, and Tom Duerig.
\newblock Scaling up visual and vision-language representation learning with
  noisy text supervision.
\newblock In \emph{ICML}, pp.\  4904--4916, 2021.

\bibitem[Karpathy \& Li(2015)Karpathy and Li]{karpathy2015deep}
Andrej Karpathy and Fei{-}Fei Li.
\newblock Deep visual-semantic alignments for generating image descriptions.
\newblock In \emph{CVPR}, pp.\  3128--3137, 2015.

\bibitem[Kingma \& Ba(2015)Kingma and Ba]{kingma2014_adam}
Diederik~P. Kingma and Jimmy Ba.
\newblock Adam: {A} method for stochastic optimization.
\newblock In \emph{ICLR}, 2015.

\bibitem[Kipf \& Welling(2017)Kipf and Welling]{kipf2017_semi}
Thomas~N. Kipf and Max Welling.
\newblock Semi-supervised classification with graph convolutional networks.
\newblock In \emph{ICLR}, 2017.

\bibitem[Kiros et~al.(2014)Kiros, Salakhutdinov, and Zemel]{kiros2014unifying}
Ryan Kiros, Ruslan Salakhutdinov, and Richard~S Zemel.
\newblock Unifying visual-semantic embeddings with multimodal neural language
  models.
\newblock \emph{arXiv preprint arXiv:1411.2539}, 2014.

\bibitem[Lee et~al.(2018)Lee, Chen, Hua, Hu, and He]{lee2018stacked}
Kuang-Huei Lee, Xi~Chen, Gang Hua, Houdong Hu, and Xiaodong He.
\newblock Stacked cross attention for image-text matching.
\newblock In \emph{ECCV}, pp.\  201--216, 2018.

\bibitem[Lee et~al.(2021)Lee, Arnab, Guadarrama, Canny, and
  Fischer]{lee2021compressive}
Kuang{-}Huei Lee, Anurag Arnab, Sergio Guadarrama, John~F. Canny, and Ian
  Fischer.
\newblock Compressive visual representations.
\newblock In \emph{NeurIPS}, pp.\  19538--19552, 2021.

\bibitem[Li et~al.(2020{\natexlab{a}})Li, Duan, Fang, Jiang, and
  Zhou]{li2019unicoder}
Gen Li, Nan Duan, Yuejian Fang, Daxin Jiang, and Ming Zhou.
\newblock Unicoder-{VL}: A universal encoder for vision and language by
  cross-modal pre-training.
\newblock In \emph{AAAI}, 2020{\natexlab{a}}.

\bibitem[Li et~al.(2021)Li, Selvaraju, Gotmare, Joty, Xiong, and
  Hoi]{li2021align}
Junnan Li, Ramprasaath Selvaraju, Akhilesh Gotmare, Shafiq Joty, Caiming Xiong,
  and Steven Chu~Hong Hoi.
\newblock Align before fuse: Vision and language representation learning with
  momentum distillation.
\newblock \emph{NeurIPS}, pp.\  9694--9705, 2021.

\bibitem[Li et~al.(2022{\natexlab{a}})Li, Li, Xiong, and Hoi]{li2022blip}
Junnan Li, Dongxu Li, Caiming Xiong, and Steven C.~H. Hoi.
\newblock {BLIP}: Bootstrapping language-image pre-training for unified
  vision-language understanding and generation.
\newblock In \emph{ICML}, pp.\  12888--12900, 2022{\natexlab{a}}.

\bibitem[Li et~al.(2023{\natexlab{a}})Li, Li, Savarese, and Hoi]{li2023blip}
Junnan Li, Dongxu Li, Silvio Savarese, and Steven C.~H. Hoi.
\newblock {BLIP-2:} bootstrapping language-image pre-training with frozen image
  encoders and large language models.
\newblock In \emph{ICML}, pp.\  19730--19742, 2023{\natexlab{a}}.

\bibitem[Li et~al.(2019{\natexlab{a}})Li, Zhang, Li, Li, and Fu]{li2019visual}
Kunpeng Li, Yulun Zhang, Kai Li, Yuanyuan Li, and Yun Fu.
\newblock Visual semantic reasoning for image-text matching.
\newblock In \emph{ICCV}, pp.\  4654--4662, 2019{\natexlab{a}}.

\bibitem[Li et~al.(2019{\natexlab{b}})Li, Yatskar, Yin, Hsieh, and
  Chang]{li2019visualbert}
Liunian~Harold Li, Mark Yatskar, Da~Yin, Cho-Jui Hsieh, and Kai-Wei Chang.
\newblock {VisualBERT}: A simple and performant baseline for vision and
  language.
\newblock \emph{arXiv preprint arXiv:1908.03557}, 2019{\natexlab{b}}.

\bibitem[Li et~al.(2023{\natexlab{b}})Li, Fan, Yuan, He, Tian, Feris, Indyk,
  and Katabi]{li2023addressing}
Tianhong Li, Lijie Fan, Yuan Yuan, Hao He, Yonglong Tian, Rog{\'{e}}rio Feris,
  Piotr Indyk, and Dina Katabi.
\newblock Addressing feature suppression in unsupervised visual
  representations.
\newblock In \emph{WACV}, pp.\  1411--1420, 2023{\natexlab{b}}.

\bibitem[Li et~al.(2020{\natexlab{b}})Li, Yin, Li, Zhang, Hu, Zhang, Wang, Hu,
  Dong, Wei, et~al.]{li2020oscar}
Xiujun Li, Xi~Yin, Chunyuan Li, Pengchuan Zhang, Xiaowei Hu, Lei Zhang, Lijuan
  Wang, Houdong Hu, Li~Dong, Furu Wei, et~al.
\newblock Oscar: Object-semantics aligned pre-training for vision-language
  tasks.
\newblock In \emph{ECCV}, pp.\  121--137, 2020{\natexlab{b}}.

\bibitem[Li et~al.(2022{\natexlab{b}})Li, Liang, Zhao, Cui, Ouyang, Shao, Yu,
  and Yan]{li2022supervision}
Yangguang Li, Feng Liang, Lichen Zhao, Yufeng Cui, Wanli Ouyang, Jing Shao,
  Fengwei Yu, and Junjie Yan.
\newblock Supervision exists everywhere: {A} data efficient contrastive
  language-image pre-training paradigm.
\newblock In \emph{ICLR}, 2022{\natexlab{b}}.

\bibitem[Liang et~al.(2023)Liang, Deng, Ma, Zou, Morency, and
  Salakhutdinov]{liang2023factorized}
Paul~Pu Liang, Zihao Deng, Martin~Q. Ma, James Zou, Louis-Philippe Morency, and
  Russ Salakhutdinov.
\newblock Factorized contrastive learning: Going beyond multi-view redundancy.
\newblock In \emph{NeurIPS}, 2023.

\bibitem[Lin et~al.(2014)Lin, Maire, Belongie, Hays, Perona, Ramanan,
  Doll{\'a}r, and Zitnick]{lin2014microsoft}
Tsung-Yi Lin, Michael Maire, Serge Belongie, James Hays, Pietro Perona, Deva
  Ramanan, Piotr Doll{\'a}r, and C~Lawrence Zitnick.
\newblock Microsoft {COCO}: Common objects in context.
\newblock In \emph{ECCV}, pp.\  740--755, 2014.

\bibitem[Loshchilov \& Hutter(2019)Loshchilov and
  Hutter]{loshchilov2019decoupled}
Ilya Loshchilov and Frank Hutter.
\newblock Decoupled weight decay regularization.
\newblock In \emph{ICLR}, 2019.

\bibitem[Lu et~al.(2019)Lu, Batra, Parikh, and Lee]{lu2019vilbert}
Jiasen Lu, Dhruv Batra, Devi Parikh, and Stefan Lee.
\newblock Vi{LBERT}: Pretraining task-agnostic visiolinguistic representations
  for vision-and-language tasks.
\newblock In \emph{NeurIPS}, pp.\  13--23, 2019.

\bibitem[Mu et~al.(2022)Mu, Kirillov, Wagner, and Xie]{mu2022slip}
Norman Mu, Alexander Kirillov, David~A. Wagner, and Saining Xie.
\newblock {SLIP:} {S}elf-supervision meets language-image pre-training.
\newblock In \emph{ECCV}, pp.\  529--544, 2022.

\bibitem[Radford et~al.(2019)Radford, Wu, Child, Luan, Amodei, Sutskever,
  et~al.]{radford2019language}
Alec Radford, Jeffrey Wu, Rewon Child, David Luan, Dario Amodei, Ilya
  Sutskever, et~al.
\newblock Language models are unsupervised multitask learners.
\newblock \emph{Open{AI} blog}, pp.\ ~9, 2019.

\bibitem[Radford et~al.(2021)Radford, Kim, Hallacy, Ramesh, Goh, Agarwal,
  Sastry, Askell, Mishkin, Clark, Krueger, and Sutskever]{radford2021learning}
Alec Radford, Jong~Wook Kim, Chris Hallacy, Aditya Ramesh, Gabriel Goh,
  Sandhini Agarwal, Girish Sastry, Amanda Askell, Pamela Mishkin, Jack Clark,
  Gretchen Krueger, and Ilya Sutskever.
\newblock Learning transferable visual models from natural language
  supervision.
\newblock In \emph{ICML}, pp.\  8748--8763, 2021.

\bibitem[Reimers \& Gurevych(2019)Reimers and Gurevych]{reimersers2019sentence}
Nils Reimers and Iryna Gurevych.
\newblock Sentence-{BERT}: Sentence embeddings using siamese {BERT}-{N}etworks.
\newblock In \emph{EMNLP-IJCNLP}, pp.\  3980--3990, 2019.

\bibitem[Robinson et~al.(2021)Robinson, Sun, Yu, Batmanghelich, Jegelka, and
  Sra]{robinson2021can}
Joshua Robinson, Li~Sun, Ke~Yu, Kayhan Batmanghelich, Stefanie Jegelka, and
  Suvrit Sra.
\newblock Can contrastive learning avoid shortcut solutions?
\newblock In \emph{NeurIPS}, pp.\  4974--4986, 2021.

\bibitem[Scimeca et~al.(2022)Scimeca, Oh, Chun, Poli, and
  Yun]{scimeca2022which}
Luca Scimeca, Seong~Joon Oh, Sanghyuk Chun, Michael Poli, and Sangdoo Yun.
\newblock Which shortcut cues will {DNN}s choose? {A} study from the
  parameter-space perspective.
\newblock In \emph{ICLR}, 2022.

\bibitem[Shwartz-Ziv \& LeCun(2023)Shwartz-Ziv and LeCun]{shwartz2023compress}
Ravid Shwartz-Ziv and Yann LeCun.
\newblock To compress or not to compress--self-supervised learning and
  information theory: A review.
\newblock \emph{arXiv preprint arXiv:2304.09355}, 2023.

\bibitem[Simonyan \& Zisserman(2015)Simonyan and
  Zisserman]{simonyan2015_very_deep}
Karen Simonyan and Andrew Zisserman.
\newblock Very deep convolutional networks for large-scale image recognition.
\newblock In \emph{ICLR}, 2015.

\bibitem[Song et~al.(2020)Song, Tan, Qin, Lu, and Liu]{song2020mpnet}
Kaitao Song, Xu~Tan, Tao Qin, Jianfeng Lu, and Tie{-}Yan Liu.
\newblock {MPN}et: Masked and permuted pre-training for language understanding.
\newblock In \emph{NeurIPS}, pp.\  16857--16867, 2020.

\bibitem[Sridharan \& Kakade(2008)Sridharan and
  Kakade]{sridharan2008information}
Karthik Sridharan and Sham~M. Kakade.
\newblock An information theoretic framework for multi-view learning.
\newblock In \emph{COLT}, pp.\  403--414, 2008.

\bibitem[Su et~al.(2020)Su, Zhu, Cao, Li, Lu, Wei, and Dai]{su2019vl}
Weijie Su, Xizhou Zhu, Yue Cao, Bin Li, Lewei Lu, Furu Wei, and Jifeng Dai.
\newblock {VL}-{BERT}: Pre-training of generic visual-linguistic
  representations.
\newblock In \emph{ICLR}, 2020.

\bibitem[Tan \& Bansal(2019)Tan and Bansal]{tan2019lxmert}
Hao Tan and Mohit Bansal.
\newblock {LXMERT:} learning cross-modality encoder representations from
  transformers.
\newblock In \emph{EMNLP-IJCNLP}, pp.\  5099--5110, 2019.

\bibitem[Tian et~al.(2020{\natexlab{a}})Tian, Krishnan, and
  Isola]{tian2020contrastive}
Yonglong Tian, Dilip Krishnan, and Phillip Isola.
\newblock Contrastive multiview coding.
\newblock In \emph{ECCV}, pp.\  776--794, 2020{\natexlab{a}}.

\bibitem[Tian et~al.(2020{\natexlab{b}})Tian, Sun, Poole, Krishnan, Schmid, and
  Isola]{tian2020what}
Yonglong Tian, Chen Sun, Ben Poole, Dilip Krishnan, Cordelia Schmid, and
  Phillip Isola.
\newblock What makes for good views for contrastive learning?
\newblock In \emph{NeurIPS}, pp.\  6827--6839, 2020{\natexlab{b}}.

\bibitem[Tsai et~al.(2021)Tsai, Wu, Salakhutdinov, and Morency]{tsai2021self}
Yao{-}Hung~Hubert Tsai, Yue Wu, Ruslan Salakhutdinov, and Louis{-}Philippe
  Morency.
\newblock Self-supervised learning from a multi-view perspective.
\newblock In \emph{ICLR}, 2021.

\bibitem[Tschannen et~al.(2020)Tschannen, Djolonga, Rubenstein, Gelly, and
  Lucic]{tschannen2020on}
Michael Tschannen, Josip Djolonga, Paul~K. Rubenstein, Sylvain Gelly, and Mario
  Lucic.
\newblock On mutual information maximization for representation learning.
\newblock In \emph{ICLR}, 2020.

\bibitem[Tschannen et~al.(2023)Tschannen, Kumar, Steiner, Zhai, Houlsby, and
  Beyer]{tschannen2023image}
Michael Tschannen, Manoj Kumar, Andreas~Peter Steiner, Xiaohua Zhai, Neil
  Houlsby, and Lucas Beyer.
\newblock Image captioners are scalable vision learners too.
\newblock In \emph{NeurIPS}, 2023.

\bibitem[van~den Oord et~al.(2018)van~den Oord, Li, and
  Vinyals]{oord2018representation}
Aaron van~den Oord, Yazhe Li, and Oriol Vinyals.
\newblock Representation learning with contrastive predictive coding.
\newblock \emph{arXiv preprint arXiv:1807.03748}, 2018.

\bibitem[Wang et~al.(2022)Wang, Guo, Deng, and Lu]{wang2022rethinking}
Haoqing Wang, Xun Guo, Zhi{-}Hong Deng, and Yan Lu.
\newblock Rethinking minimal sufficient representation in contrastive learning.
\newblock In \emph{CVPR}, pp.\  16020--16029, 2022.

\bibitem[Wang et~al.(2023)Wang, Bao, Dong, Bjorck, Peng, Liu, Aggarwal,
  Mohammed, Singhal, Som, and Wei]{wang2022image}
Wenhui Wang, Hangbo Bao, Li~Dong, Johan Bjorck, Zhiliang Peng, Qiang Liu, Kriti
  Aggarwal, Owais~Khan Mohammed, Saksham Singhal, Subhojit Som, and Furu Wei.
\newblock Image as a foreign language: {BEIT} pretraining for vision and
  vision-language tasks.
\newblock In \emph{CVPR}, pp.\  19175--19186, 2023.

\bibitem[Wiles et~al.(2022)Wiles, Gowal, Stimberg, Rebuffi, Ktena, Dvijotham,
  and Cemgil]{wiles_2022_afinegrained}
Olivia Wiles, Sven Gowal, Florian Stimberg, Sylvestre{-}Alvise Rebuffi, Ira
  Ktena, Krishnamurthy Dvijotham, and Ali~Taylan Cemgil.
\newblock A fine-grained analysis on distribution shift.
\newblock In \emph{ICLR}, 2022.

\bibitem[Xiao et~al.(2021)Xiao, Wang, Efros, and Darrell]{xiao2021what}
Tete Xiao, Xiaolong Wang, Alexei~A. Efros, and Trevor Darrell.
\newblock What should not be contrastive in contrastive learning.
\newblock In \emph{ICLR}, 2021.

\bibitem[Yao et~al.(2022)Yao, Huang, Hou, Lu, Niu, Xu, Liang, Li, Jiang, and
  Xu]{yao2022flip}
Lewei Yao, Runhui Huang, Lu~Hou, Guansong Lu, Minzhe Niu, Hang Xu, Xiaodan
  Liang, Zhenguo Li, Xin Jiang, and Chunjing Xu.
\newblock {FILIP:} {F}ine-grained interactive language-image pre-training.
\newblock In \emph{ICLR}, 2022.

\bibitem[Young et~al.(2014)Young, Lai, Hodosh, and Hockenmaier]{young2014image}
Peter Young, Alice Lai, Micah Hodosh, and Julia Hockenmaier.
\newblock From image descriptions to visual denotations: New similarity metrics
  for semantic inference over event descriptions.
\newblock \emph{Transactions of the Association for Computational Linguistics},
  2:\penalty0 67--78, 2014.

\bibitem[Yuan et~al.(2021)Yuan, Chen, Chen, Codella, Dai, Gao, Hu, Huang, Li,
  Li, et~al.]{yuan2021florence}
Lu~Yuan, Dongdong Chen, Yi-Ling Chen, Noel Codella, Xiyang Dai, Jianfeng Gao,
  Houdong Hu, Xuedong Huang, Boxin Li, Chunyuan Li, et~al.
\newblock Florence: A new foundation model for computer vision.
\newblock \emph{arXiv preprint arXiv:2111.11432}, 2021.

\bibitem[Yuksekgonul et~al.(2023)Yuksekgonul, Bianchi, Kalluri, Jurafsky, and
  Zou]{yuksekgonul2023when}
Mert Yuksekgonul, Federico Bianchi, Pratyusha Kalluri, Dan Jurafsky, and James
  Zou.
\newblock When and why vision-language models behave like bags-of-words, and
  what to do about it?
\newblock In \emph{ICLR}, 2023.

\bibitem[Zeng et~al.(2022)Zeng, Zhang, and Li]{zeng2022multi}
Yan Zeng, Xinsong Zhang, and Hang Li.
\newblock Multi-grained vision language pre-training: Aligning texts with
  visual concepts.
\newblock In \emph{ICML}, pp.\  25994--26009, 2022.

\bibitem[Zhao et~al.(2017)Zhao, Xie, Xu, and Sun]{zhao2017multi}
Jing Zhao, Xijiong Xie, Xin Xu, and Shiliang Sun.
\newblock Multi-view learning overview: Recent progress and new challenges.
\newblock \emph{Inf. Fusion}, 38:\penalty0 43--54, 2017.

\bibitem[Zong et~al.(2023)Zong, Mac~Aodha, and Hospedales]{zong2023self}
Yongshuo Zong, Oisin Mac~Aodha, and Timothy Hospedales.
\newblock Self-supervised multimodal learning: A survey.
\newblock \emph{arXiv preprint arXiv:2304.01008}, 2023.

\end{thebibliography}
